\documentclass[journal,comsoc]{IEEEtran}
\usepackage[utf8]{inputenc} 
\usepackage[T1]{fontenc}
\usepackage{url}
 
\usepackage[cmex10]{amsmath}
\usepackage{array} % Use the [cmex10] option to ensure complicance
\usepackage{bbding}
\usepackage[normal]{caption}
\renewcommand{\raggedright}{\leftskip=0pt \rightskip=0pt plus 0cm} % so amazing!!!
\usepackage[ruled,vlined]{algorithm2e}
\usepackage{bm}
\usepackage{tikz}
\usetikzlibrary{arrows}
\usepackage{graphicx,booktabs,multirow}
\usepackage{epstopdf}
\usepackage[FIGTOPCAP]{subfigure}
\usepackage{float}
\usepackage{multirow}
\usepackage{tabularx} 
\usepackage{booktabs}

\definecolor{colorhkust}{RGB}{20,43,140}
\definecolor{colortsinghua}{RGB}{116,52,129}
\definecolor{color1}{RGB}{128,0,0}

\usepackage{makecell}
\usepackage{amsmath,amssymb,amsfonts,amsthm}
\usepackage{cite}
\usepackage{hyperref} 
\usepackage{bbm}
\theoremstyle{plain}
\newtheorem{theorem}{\bf{Theorem}}  
\newtheorem{lemma}{\bf{Lemma}}  
  
\newtheorem{assumption}{\bf{Assumption}}  
\theoremstyle{remark}
\newtheorem{remark}{\bf{Remark}}
\newtheorem{corollary}{\bf{Corollary}}
\newcommand{\TT}{{\mathrm{T}}}
\usepackage{enumitem}
\setlist{leftmargin=1em}
\setitemize[1]{itemsep=0pt,partopsep=0pt,parsep=\parskip,topsep=0pt}
\begin{document}

    \title{Semi-Decentralized Federated Edge Learning with Data and Device Heterogeneity} 
   
    \author{
        Yuchang~Sun,~\IEEEmembership{Graduate Student Member,~IEEE},
        Jiawei~Shao,~\IEEEmembership{Graduate Student Member,~IEEE},
        Yuyi~Mao,~\IEEEmembership{Member,~IEEE,}
        Jessie Hui~Wang,~\IEEEmembership{Member,~IEEE,}
        and~Jun~Zhang,~\IEEEmembership{Fellow,~IEEE}% <-this % stops a space
      	\thanks{
      	Parts of this paper were presented at the 2022 IEEE Wireless Communications and Networking Conference \cite{sun2021semi} and the 2022 IEEE International Conference on Communications \cite{sun2022icc}. \emph{(Corresponding author: Yuyi Mao.)}
      	
      	Y. Sun, J. Shao, and J. Zhang are with the Department of Electronic and Computer Engineering, the Hong Kong University of Science and Technology, Hong Kong (E-mail: \{yuchang.sun, jiawei.shao\}@connect.ust.hk, eejzhang@ust.hk). 
      	Y. Mao is with the Department of Electronic and Information Engineering, the Hong Kong Polytechnic University, Hong Kong (E-mail: yuyi-eie.mao@polyu.edu.hk).
       J. H. Wang is with the Institute for Network Sciences and Cyberspace, Tsinghua University, Beijing 100084, and ZGC Lab, Beijing 100194, China. (E-mail: jessiewang@tsinghua.edu.cn).}
}

    \maketitle

\begin{abstract}

Federated edge learning (FEEL) emerges as a privacy-preserving paradigm to effectively train deep learning models from the distributed data in 6G networks.
Nevertheless, the limited coverage of a single edge server results in an insufficient number of participating client nodes, which may impair the learning performance.
In this paper, we investigate a novel FEEL framework, namely \emph{semi-decentralized federated edge learning} (SD-FEEL), where multiple edge servers collectively coordinate a large number of client nodes. 
By exploiting the low-latency communication among edge servers for efficient model sharing, SD-FEEL incorporates more training data, while enjoying lower latency compared with conventional federated learning.
We detail the training algorithm for SD-FEEL with three steps, including local model update, intra-cluster, and inter-cluster model aggregations. 
The convergence of this algorithm is proved on non-independent and identically distributed data, which reveals the effects of key parameters and provides design guidelines.
Meanwhile, the heterogeneity of edge devices may cause the straggler effect and deteriorate the convergence speed of SD-FEEL.
To resolve this issue, we propose an asynchronous training algorithm with a staleness-aware aggregation scheme, of which, the convergence is also analyzed.
The simulations demonstrate the effectiveness and efficiency of the proposed algorithms for SD-FEEL and corroborate our analysis.
\end{abstract}


%%%%%%%%%%%%%%%%%%%%%%%%%%%%%%%%%%%%%%%%%%
%%%%%%%%%%%%%% Introduction %%%%%%%%%%%%%%
%%%%%%%%%%%%%%%%%%%%%%%%%%%%%%%%%%%%%%%%%%
\section{Introduction}
The recent upsurge of Internet of Things (IoT) applications brings about a drastically increasing number of IoT devices, which is predicted to reach more than 30 billion by 2025 \cite{iot}. 
As a result, an unprecedented volume of data is generated and stored at the edge of the wireless network, which can facilitate the training of powerful machine learning (ML) models to empower various intelligent mobile applications.
Meanwhile, as an enabler of emerging IoT applications, the sixth generation (6G) of wireless networks is envisioned to provide ubiquitous artificial intelligence (AI) services anywhere at any time \cite{letaief2019roadmap,6giot}.
To leverage the valuable data resources, a traditional approach is to upload them to a centralized server for training.
However, this approach may not be suitable to support ubiquitous AI in 6G networks, since data uploading incurs heavy communication overhead and serious concerns about privacy leakage \cite{meneghello2019iot}.
Therefore, there is an emerging trend of pushing AI computing to edge devices \cite{shi2020communication}.

In 2017, Google proposed a privacy-preserving training paradigm, namely federated learning (FL) \cite{mcmahan2017communication}, where the client nodes (e.g., mobile and IoT devices) train models based on local data and periodically upload them to a Cloud-based parameter server (PS) for model aggregation. 
Since FL requires no data sharing, it substantially avoids privacy leakage.
Nevertheless, model uploading between the client nodes and the Cloud-based PS introduces expensive communication costs, which degrades the efficiency of FL.
Inspired by the emerging mobile edge computing (MEC) platforms \cite{mao2017survey}, federated edge learning (FEEL) \cite{lim2020federated} has been proposed to overcome this bottleneck, where an edge-based PS (i.e., the edge server) is deployed to be located near the edge devices as a model aggregator. 
Despite its great promise in reducing the model uploading latency, the training efficiency of FEEL is far from satisfactory, since the number of client nodes accessible by a single edge server is insufficient due to its limited coverage.

To fully exploit the potential of FEEL, recent works started to incorporate multiple edge servers, each of which is in charge of a number of client nodes. Accordingly, the total amount of training data samples can be significantly increased. To speed up the training process, edge servers collaborate in training by sharing their models. A possible solution is to allow the Cloud to collect and aggregate the models from edge servers \cite{HierFL,wang2020local}, which still introduces excessive communication latency.
In this work, we will investigate a novel FL architecture to support collaborated learning in 6G networks, namely semi-decentralized federated edge learning (SD-FEEL) \cite{castiglia2020multi,sun2021semi}, which utilizes low-latency communication among edge servers to realize effective and efficient model aggregation.

\subsection{Related Works}
%%%%%%% comm %%%%%%
The training efficiency of FEEL faces two bottlenecks, i.e., the limited communication bandwidth and the straggler effect.
On one hand, the client nodes and edge servers suffer from unstable wireless connections and thus frequent communications will cause a large training latency.
To improve the learning efficiency, a control algorithm was proposed in \cite{wang2019adaptive} to adaptively determine the global model aggregation frequency given the available resources on client nodes. 
Besides, Shi \textit{et} \textit{al.} \cite{shi2020joint} solved a joint bandwidth allocation and device scheduling problem to maximize the learning performance of FL within the given training time. Moreover, gradient quantization and sparsification techniques were adopted to achieve communication-efficient FEEL \cite{mills2019communication}, \cite{amiri2020federated}.

%%%%%%% comp straggler %%%%%%
On the other hand, different types of edge devices have heterogeneous computational resources, e.g., processing speed, battery capacity, and memory usage \cite{lim2020federated}.
Particularly, it may take a longer time for the client nodes with less computational resources (namely stragglers) to conduct the same amount of local training, which prolongs the total training time.
This device heterogeneity issue can be problematic especially in large-scale implementations. 
A client selection algorithm for FEEL was proposed in \cite{nishio2019client}, which eliminates the straggling client nodes from global model aggregation. 
Meanwhile, asynchronous FL has attracted much attention due to its effectiveness in dealing with device heterogeneity \cite{xie2019asynchronous,ma2021fedsa,wu2020safa}. 
Xie \textit{et} \textit{al.} \cite{xie2019asynchronous} proposed the first asynchronous training algorithm for FL, where the model aggregation is triggered once the PS receives an update from any client node. However, this scheme incurs more frequent communications between the PS and client nodes. To strike a balance between model improvement and training latency, a semi-asynchronous training algorithm for FL was proposed in \cite{ma2021fedsa}, where the model aggregation is delayed until the PS collects a targeted number of local updates. However, the stale models uploaded from the straggling client nodes are less valuable to the global model aggregation and may even degrade the learning performance. The design in \cite{wu2020safa} relieved this issue by forcing some client nodes with up-to-date or deprecated local models to synchronize with the PS, while most client nodes stay asynchronous.

%%%%%%% Multiple cluster Architectures %%%%%%
When being implemented over wireless networks, the aforementioned benefits of FEEL are hindered by the limited coverage of a single edge server. 
Recent works considered the cooperation among multiple edge servers in training to further explore the potential of FEEL. 
The client-edge-cloud hierarchical FL system, namely HierFAVG, was investigated in \cite{HierFL,wang2020local}, where each edge server is responsible for aggregating the models of its associated client nodes, while the Cloud-based PS performs global aggregation to average the models from edge servers periodically.
However, the communication with the Cloud-based PS still incurs a high latency.
As an alternative to the Cloud-based aggregator, in \cite{saha2020fogfl}, a fog node (i.e., an edge server) was selected by the Cloud to perform global model aggregation in each communication round.
Such a design requires all the fog nodes to be fully connected and may suffer from single-point failure.
Besides, a recent work \cite{han2021fedmes} adopted the cell-edge users (i.e., client nodes) as bridges to share models between multiple edge servers.
Nevertheless, these client nodes may suffer from poor signal quality because of the cell-edge effect \cite{you2011cell}, which degrades the training performance.
Table \ref{tab:related} summarizes the main characteristics of the above works.
\begin{table*}[!t]
\centering
\linespread{1.2} 
\caption{Comparison of FEEL systems with multiple edge servers}
\resizebox{\textwidth}{!}{
\begin{tabular}{|c|c|c|c|c|c|c|}
\hline
  \multicolumn{1}{|c|}{\textbf{System}} &
  \multicolumn{1}{c|}{\textbf{\begin{tabular}[c]{@{}c@{}}Communication \\ Among \\ Edge Servers \end{tabular}}} &
  \multicolumn{1}{c|}{\textbf{\begin{tabular}[c]{@{}c@{}}Communication \\ Between the Cloud \\ and Edge Servers \end{tabular}}} &
  \multicolumn{1}{c|}{\textbf{Latency}} &
  \textbf{\begin{tabular}[c]{@{}c@{}}Insensitive to \\ Network \\ Topology \end{tabular}} &
  \multicolumn{1}{c|}{\textbf{\begin{tabular}[c]{@{}c@{}}Insensitive to \\ the Cell-edge \\ Effect \end{tabular}}} &
  \multicolumn{1}{c|}{\textbf{\begin{tabular}[c]{@{}c@{}} Convergence \\ Guarantee on \\ Non-IID Data \end{tabular}}} \\ \hline
HierFAVG \cite{HierFL,wang2020local} &
  \multicolumn{1}{c|}{\XSolidBrush} &
  \multicolumn{1}{c|}{\Checkmark} &
  \multicolumn{1}{c|}{High} &
  \multicolumn{1}{c|}{\Checkmark} &
  \multicolumn{1}{c|}{\Checkmark} &
  \multicolumn{1}{c|}{\Checkmark} \\ \hline
FogFL \cite{saha2020fogfl} &
  \multicolumn{1}{c|}{\Checkmark} &
  \multicolumn{1}{c|}{\Checkmark} &
  \multicolumn{1}{c|}{High} &
  \multicolumn{1}{c|}{\XSolidBrush} &
  \multicolumn{1}{c|}{\Checkmark} &
  \multicolumn{1}{c|}{\XSolidBrush} \\ \hline
FedMes \cite{han2021fedmes} &
  \multicolumn{1}{c|}{\XSolidBrush} &
  \multicolumn{1}{c|}{\XSolidBrush} &
  \multicolumn{1}{c|}{Low} &
  \multicolumn{1}{c|}{\Checkmark} &
  \multicolumn{1}{c|}{\XSolidBrush} &
  \multicolumn{1}{c|}{\XSolidBrush} \\ \hline
Multi-level Local SGD \cite{castiglia2020multi} &
  \multicolumn{1}{c|}{\Checkmark} &
  \multicolumn{1}{c|}{\XSolidBrush} &
  \multicolumn{1}{c|}{Low} &
  \multicolumn{1}{c|}{\XSolidBrush} &
  \multicolumn{1}{c|}{\Checkmark} &
  \multicolumn{1}{c|}{\XSolidBrush} \\ \hline
SD-FEEL (Ours) &
  \multicolumn{1}{c|}{\Checkmark} &
  \multicolumn{1}{c|}{\XSolidBrush} &
  \multicolumn{1}{c|}{Low} &
  \multicolumn{1}{c|}{\Checkmark} &
  \multicolumn{1}{c|}{\Checkmark} &
  \multicolumn{1}{c|}{\Checkmark} \\ \hline
\end{tabular}}
\label{tab:related}
\end{table*}

Motivated by the efficient communication among edge servers, this paper investigates a novel FEEL system, \emph{semi-decentralized federated edge learning} (SD-FEEL) \cite{castiglia2020multi,sun2021semi}, where multiple edge servers collectively coordinate a large number of client nodes. The client nodes perform local model updates for several iterations and then upload the model to the associated edge server for \textit{intra-cluster} model aggregation. After several times of intra-cluster model aggregations, multiple rounds of model exchanges and aggregations with the neighboring edge servers, which is termed \textit{inter-cluster} model aggregation, are conducted by edge servers. Such a semi-decentralized architecture can easily include a huge amount of client nodes and adequately explore their data at a very low cost. 
The recent work \cite{castiglia2020multi} investigated a similar design as SD-FEEL and analyzed its convergence, but only on independent and identically distributed (IID) data. Besides, they assumed only one round of communication among edge servers, which may degrade the training performance due to the model inconsistency among different edge servers.

\subsection{Main Contributions}
This paper investigates SD-FEEL at the edge of 6G wireless networks, accounting for the non-IID data and heterogeneous computational resources at different client nodes.
Our main contributions are summarized as follows:
\begin{itemize}
    \item We first propose a generic training algorithm for SD-FEEL, where client nodes and edge servers collaboratively train an ML model through local updates, intra-cluster, and inter-cluster model aggregations. To relieve the model inconsistency among edge clusters, edge servers are allowed to exchange and aggregate models multiple times in each round of inter-cluster model aggregation.
    \item We analyze the convergence of the training algorithm on non-IID data, with data heterogeneity among clients within the same edge cluster, as well as that among different edge clusters.
    Based on the analytical results, the impacts of intra-/inter-cluster aggregation periods are discussed.
    We also investigate how the network topology among edge servers affects the convergence rate.
    It is found that SD-FEEL converges slowly when the edge servers are sparsely connected, which, however, can be alleviated by multiple times of model sharings in inter-cluster model aggregation.
    \item To effectively combat device heterogeneity that may significantly hinder the training performance, we propose an asynchronous training algorithm for SD-FEEL, where edge servers can independently set deadlines for local computation at client nodes. Particularly, we design a staleness-aware aggregation scheme to account for device heterogeneity. 
    For analysis, we decompose the variance incurred by asynchronous training and prove the convergence of the proposed algorithm.
    \item We conduct extensive simulations on two image classification tasks. The results demonstrate the benefits of SD-FEEL in achieving faster convergence without sacrificing model quality.
    The simulations also verify our discussions on the effects of various critical parameters.
    Especially, increasing the number of inter-server communication rounds is found to be an effective strategy to reduce the model inconsistency among edge clusters.
    Besides, further experiments show that the proposed asynchronous training algorithm improves the test accuracy when SD-FEEL is under a high degree of device heterogeneity.
\end{itemize}
\subsection{Organizations}
The rest of the paper is organized as follows.
In Section \ref{sec:system}, we introduce the SD-FEEL system and propose a generic training algorithm. Section \ref{sec:theory} presents the convergence analysis and corresponding discussions, including the effects of key parameters and comparison of different FEEL systems. In Section \ref{sec:async}, we investigate SD-FEEL with device heterogeneity and propose an asynchronous training algorithm, followed by its convergence analysis. We provide the simulation results in Section \ref{sec:experiment} and conclude this paper in Section \ref{sec:conclusion}.

\subsection{Notations}
Throughout this paper, we use bold-face lower-case letters, bold-face upper-case letters, and math calligraphy letters to denote vectors, matrices, and sets, respectively.
Besides, $\mathbf{X} \triangleq [\bm{x}^{(i)}]_{i=1}^N$ represents a matrix, of which, the $i$-th column vector is $\bm{x}^{(i)}$.
The $N\times N$ identity matrix is denoted by $\mathbf{I}_N$, and we define $\mathbf{1}_N\triangleq [1,1,\dots,1]_{1\times N}$.
For any $M\times N$ matrix $\mathbf{X}$, $\lambda_i(\mathbf{X})$ represents its $i$-th largest eigenvalue, the operator norm is denoted by $\left\|\mathbf{X}\right\|_{\mathrm{op}} \triangleq \max_{\left\|\bm{w}\right\|=1} \mathbf{X}\bm{w} = \sqrt{\lambda_\mathrm{max}(\mathbf{X}^\TT\mathbf{X})}$, and the weighted Frobenius norm is defined as $\left\|\mathbf{X}\right\|_{\mathbf{M}} \triangleq \sum_{i=1}^M\sum_{j=1}^N m^{i,j} |x^{i,j}|^2$, where $\mathbf{M} \triangleq \{m^{i,j}\}$.
In addition, $\lceil \cdot \rceil$ denotes the ceil function, and $\mathbbm{1}\{\cdot\}$ is the binary indicator function.
We denote $f(x)=\mathcal{O}(g(x))$ when there exists a positive real number $M$ and a real number $x_0$ such that $|f(x)|\leq Mg(x), \forall x \geq x_0$.

%%%%%%%%%%%%%%%%%%%%%%%%%%%%%%%%%%%%%%%%%%
%%%%%%%%%%%%%% System Model %%%%%%%%%%%%%%
%%%%%%%%%%%%%%%%%%%%%%%%%%%%%%%%%%%%%%%%%%
\section{System Model}\label{sec:system}
\begin{figure}[!t]
    \centering{
    \includegraphics[width=\linewidth]{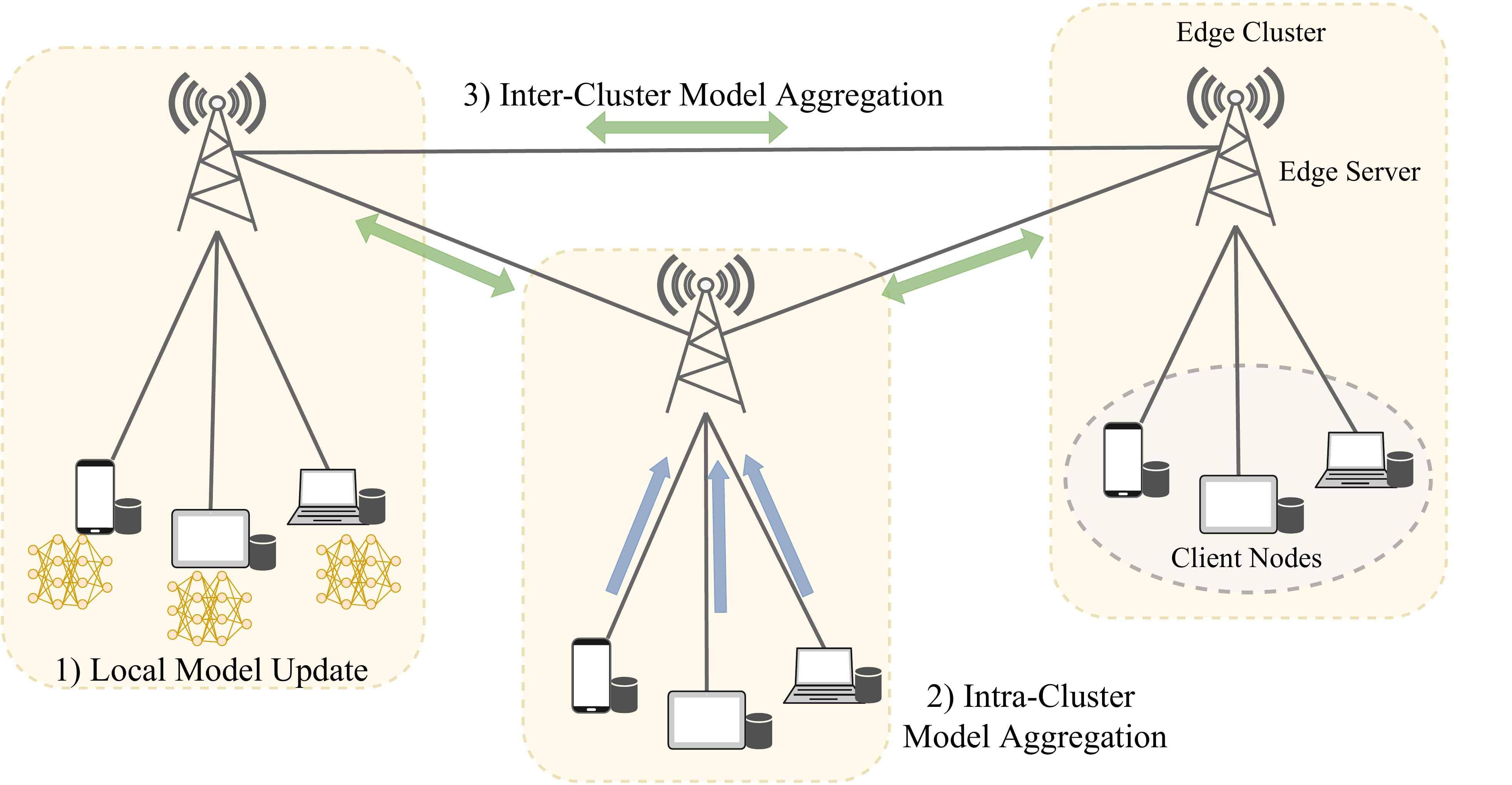}}
    \caption{The semi-decentralized FEEL system.}
    \label{fig:system}
\end{figure}

\subsection{Semi-Decentralized FEEL System}
The SD-FEEL system comprises $C$ client nodes (denoted by set $\mathcal{C}$) and $D$ edge servers (denoted by set $\mathcal{D}$), as shown in Fig. \ref{fig:system}. The system can be viewed as $D$ edge clusters, each of which consists of an edge server (denoted by $d$) and a set of associated client nodes (denoted by $\mathcal{C}_d$).
Assume each edge server coordinates at least one client node, and each client node is associated with only one edge server, according to some predefined criteria, e.g., physical proximity and network coverage. Besides, the edge servers are connected with the neighboring servers $\mathcal{N}_d$ via high-speed cables, formulating a connected graph $\mathcal{G}$.

Client node $i$ possesses a set of local training data, denoted by $\mathcal{S}_i=\{\bm{s}_{j}^{(i)}\}_{j=1}^{|\mathcal{S}_i|}$, where $\bm{s}_{j}^{(i)}$ is the $j$-th data sample at client node $i$. The collection of data samples at the set of client nodes $\mathcal{C}_{d}$ is denoted by $\Tilde{\mathcal{S}}_d$, and the training data at all the client nodes is denoted by $\mathcal{S}$. The ratios of data samples are defined as $\hat{m}_i \triangleq \frac{|\mathcal{S}_i|}{|\Tilde{\mathcal{S}}_d|}$, $m_i \triangleq \frac{|\mathcal{S}_i|}{|\mathcal{S}|}$, and $\Tilde{m}_d \triangleq \frac{|\Tilde{\mathcal{S}}_d|}{|\mathcal{S}|}$, respectively.
The client nodes collaboratively train an ML model $\bm{w}\in\mathbb{R}^M$ with $M$ trainable parameters. 
Denote the loss function associated with a data samples $\bm{s}$ with model $\bm{w}$ by $f(\bm{s}; \bm{w})$, a typical example of which is the categorical cross-entropy between the predicted label and the ground truth for a classification task.
Accordingly, the objective of SD-FEEL is to minimize the value of the loss function over the training data across all the client nodes, i.e., $\min_{\bm{w}\in\mathbb{R}^{M}} \left\{ F(\bm{w})\triangleq \sum_{i\in\mathcal{C}} \frac{|\mathcal{S}_i|}{|\mathcal{S}|} F_i(\bm{w}) \right\}$, 
where $F_i(\bm{w}) \triangleq \frac{1}{|\mathcal{S}_i|} \sum_{j\in\mathcal{S}_i} f(\bm{s}_j^{(i)}; \bm{w})$ is the local loss at client node $i$.

We use $h_i$ to denote the computational speed of client node $i$, which is in the unit of floating point operations per second (i.e., FLOPS).
Accordingly, the degree of device heterogeneity is characterized by the \emph{heterogeneity gap} $H \!\triangleq\! \max_{i,j\in\mathcal{C}}\frac{h_i}{h_j}$.

\subsection{Training Algorithm}
Assume there are $K$ iterations in the training process, which includes three main procedures: 1) local model update, 2) intra-cluster model aggregation, and 3) inter-cluster model aggregation. 
While local model update takes place at every training iteration, intra-/inter-cluster model aggregations are triggered with periods of $\tau_1$ and $\tau_1\tau_2$ training iterations, respectively, where both $\tau_1$ and $\tau_2$ are positive integers. % \in \mathbb{Z}^{+}
We show an illustration of SD-FEEL training for two edge clusters in Fig. \ref{fig:async-2cluster} and introduce the details as follows.
For clarity, we will first present synchronous SD-FEEL, while the asynchronous training will be considered in Section \ref{sec:async}.

\subsubsection{Local Model Update}
Denote the model on the client node $i$ at the beginning of the $k$-th training iteration by $\bm{w}_{k-1}^{(i)}$. This client node performs model updating based on its local data by using the mini-batch stochastic gradient descent (SGD) algorithm \cite{dekel2012optimal}, which is expressed as follows:
\begin{equation}
    \bm{w}_{k}^{(i)} \leftarrow \bm{w}_{k-1}^{(i)} - \eta g(\bm{\xi}_{k}^{(i)}; \bm{w}_{k-1}^{(i)}), i\in\mathcal{C},
\label{eq:update}
\end{equation}
where $\eta$ is the learning rate, and $g(\bm{\xi}_{k}^{(i)}; \bm{w}_{k-1}^{(i)})$ is the stochastic gradient computed on a randomly-sampled batch of training data $\bm{\xi}_{k}^{(i)}$. % \{ s_1^{(i)}, s_2^{(i)}, \dots \}

\subsubsection{Intra-cluster Model Aggregation}
When the iteration index $k$ is an integer multiple of $\tau_1$, the client nodes upload their local models to the corresponding edge server after completing local training. The edge server aggregates these received models by computing a weighted sum as follows:
\begin{equation}
    \bm{\hat{y}}_{k}^{(d)} \leftarrow \sum_{i\in \mathcal{C}_d} \hat{m}_i \bm{w}_{k}^{(i)}, d\in\mathcal{D}.
\label{eq:intra}
\end{equation}
If $k$ is not an integer multiple of $\tau_1\tau_2$, the edge server directly sends the model $\bm{y}_{k}^{(d)}$ to its associated client nodes, i.e.,
\begin{equation}
    \bm{w}_{k+1}^{(i)} \leftarrow \bm{y}_{k}^{(d)}, i \in \mathcal{C}_d.
\label{eq:broadcast}
\end{equation}
Otherwise, the inter-cluster model aggregation is triggered.

\begin{figure*}[t]
    \centering
    \includegraphics[width=\linewidth]{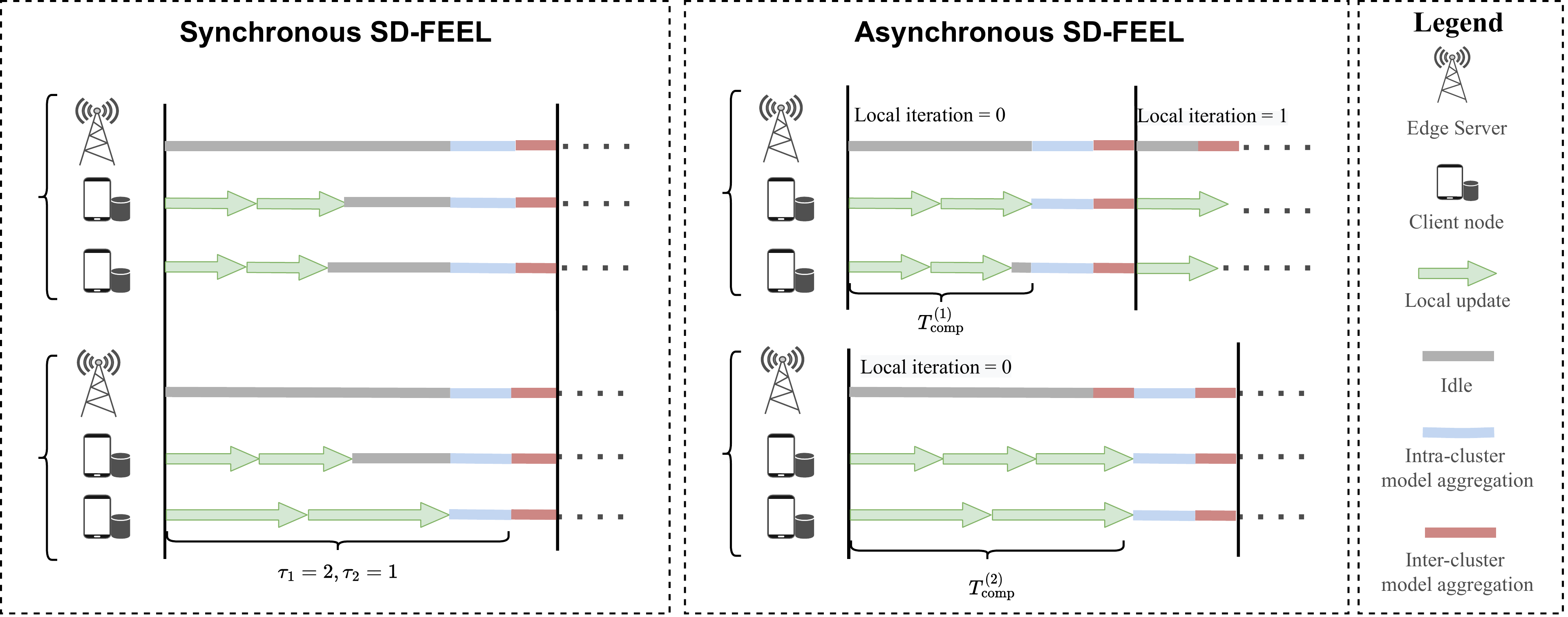}
    \caption{An illustration of synchronous (left) and asynchronous (right) SD-FEEL. In synchronous SD-FEEL, the client nodes are required to perform the same number of local iterations before intra-cluster and inter-cluster model aggregations, where the fast client nodes stay idle until all the client nodes complete their local training. In asynchronous SD-FEEL, the client nodes in edge cluster $d\in\{1,2\}$ perform local model updates for a duration of $T_\text{comp}^{(d)}$ before uploading the model updates to the associated edge server. The edge server then aggregates the received models from the client nodes in its cluster and shares the aggregated model with its one-hop neighbors.}
    \label{fig:async-2cluster}
\end{figure*}

\subsubsection{Inter-cluster Model Aggregation} When $k$ is an integer multiple of $\tau_1\tau_2$, each edge server shares its model with the one-hop neighboring edge servers after intra-cluster model aggregation. Each round of inter-cluster model aggregation includes $\alpha \in \{1,2,\dots \}$ times of model exchanges and aggregations, which can be expressed as follows:
\begin{equation}
    \bm{\hat{y}}_{k,l}^{(d)} \leftarrow \sum_{j\in \mathcal{N}_d\cup\{d\}} p^{j,d} \bm{\hat{y}}_{k,l-1}^{(j)}, l=1,2,\dots,\alpha, d\in\mathcal{D},
    \label{eq:inter}
\end{equation}
where $\bm{\hat{y}}_{k,0}^{(d)}=\bm{y}_{k}^{(d)}$ is the intra-cluster aggregated model, and we denote the mixing matrix by $\mathbf{P} \triangleq [p^{j,d}]\in \mathbb{R}^{D \times D}$.
To reduce model inconsistency among edge clusters, i.e., to ensure fast convergence to the weighted sum of the distributed models through inter-cluster model aggregation, $\mathbf{P}$ can be chosen as follows \cite{elsasser2002diffusion}:
\begin{equation}
    \mathbf{P} = \mathbf{I}_D-\frac{2}{\lambda_1(\mathbf{\Tilde{L}})+\lambda_{D-1}(\mathbf{\Tilde{L}})}\mathbf{\Tilde{L}}, 
    \label{eq5}
\end{equation}
where $\mathbf{\Tilde{L}} \triangleq \mathbf{L}\mathbf{\Omega}^{-1}$, $\mathbf{L}$ is the Laplacian matrix of graph $\mathcal{G}$, and $\mathbf{\Omega} \triangleq \mathrm{diag}(\Tilde{m}_1,\Tilde{m}_2,\dots,\Tilde{m}_D)$. We define $\zeta \!\triangleq\! |\lambda_2(\mathbf{P})| \!\in\! [0,1)$.
The edge server then updates $\bm{y}_{k}^{(d)}$ as $\bm{\hat{y}}_{k,\alpha}^{(d)}$ and broadcasts it to the associated client nodes according to \eqref{eq:broadcast}. 

After repeating the above steps for $K$ iterations ($K$ is assumed as an integer multiple of $\tau_1\tau_2$), the system enters a consensus phase where the edge servers exchange and aggregate models with their neighboring clusters. 
After sufficient rounds of such operations, the system will output a model $\sum_{d\in\mathcal{D}} \Tilde{m}_d \bm{y}_{K}^{(d)}$ and broadcast it to the associated client nodes.
The consensus phase takes place only once, which will incur negligible extra overhead.
We summarize the training algorithm of SD-FEEL in Algorithm \ref{alg-1}.

%%%%%%%%%%% Algorithm sync %%%%%%%%%%%
\begin{algorithm}[t]
\caption{Training Algorithm for SD-FEEL} \label{alg-1}
Initialize all client nodes with the same model (i.e., $\bm{w}_{0}^{(i)} \!=\! \bm{w}_{0}, \, \forall i\in\mathcal{C}$)\;
\For{$k=1,2,\dots,K$}{
\For{each client node $i\in \mathcal{C}$ in parallel}{
    Update the local model as $\bm{w}_{k}^{\left(i\right)}$ according to \eqref{eq:update}\;
    \If{$\mod(k,\tau_1)=0$}{
        \For{each edge server $d\in \mathcal{D}$ in parallel}{
        Receive the most updated model from the client nodes in $\mathcal{C}_{d}$\;
        Obtain $\bm{y}_{k}^{(d)}$ by performing intra-cluster model aggregation according to \eqref{eq:intra}\;
        \If{$\mod(k,\tau_1\tau_2)=0$}{
            Set $\bm{\hat{y}}_{k,0}^{\left(d\right)}$ as $\bm{y}_{k}^{(d)}$\;
            \For{$l=1,\dots,\alpha$}{
                Share models with $\mathcal{N}_{d}$ and perform inter-cluster model aggregation according to \eqref{eq:inter}\;
                Update $\bm{y}_{k}^{(d)}$ as $\bm{\hat{y}}_{k,\alpha}^{(d)}$ \;}
            }
        Broadcast $\bm{y}^{\left(d\right)}_{k}$ to the client nodes in $\mathcal{C}_{d}$\;}
    }
}
}
% }
  Enter the consensus phase\;
  \KwRet $\sum_{d\in\mathcal{D}} \Tilde{m}_d \bm{y}_{K}^{(d)}$\;
\end{algorithm}
% \vspace{-8pt}

%%%%%%%%%%%%%%%%%%%%%%%%%%%%%%%%%%%%%%%%%%
%%%%%%%%% Theoretical Analysis %%%%%%%%%%%
%%%%%%%%%%%%%%%%%%%%%%%%%%%%%%%%%%%%%%%%%%
\section{Theoretical Analysis}\label{sec:theory}
\subsection{Convergence Analysis}\label{sec:convergence}

To facilitate the convergence analysis, we make the following assumptions that are commonly used in the existing literature \cite{zhangweiyi,wang2020tackling,wang2021quantized}.
\begin{assumption}\label{assumptions}
For all $i\in \mathcal{C}$, we assume:
\begin{itemize}
    \item (\textbf{Smoothness}) The local objective function is $L$-smooth, i.e., 
        \begin{equation}
            \left\| \nabla F_i( \bm{w} ) - \nabla  F_i( \bm{w}^\prime ) \right\|_2  \leq L \left\| \bm{w} - \bm{w}^\prime \right\|_2, \forall \bm{w}, \bm{w}^\prime \in \mathbb{R}^{M} .
            \label{eq-smooth}
        \end{equation}
    \item (\textbf{Unbiased and bounded gradient variance}) The mini-batch gradient is unbiased, i.e.,
        \begin{equation}
            \mathbb{E}_{\bm{\xi}^{(i)}|\bm{w}} [g(\bm{\xi}^{(i)};\bm{w})] = \nabla F_i(\bm{w}), \forall \bm{w}\in\mathbb{R}^M,
        \label{eq-gradient}
        \end{equation}
        and there exists $\sigma>0$ such that 
        \begin{equation}
            \mathbb{E}_{\bm{\xi}^{(i)}|\bm{w}} \left[\left\| g(\bm{\xi}^{(i)};\bm{w}) - \nabla F_i(\bm{w}) \right\|_2^2\right] \leq \sigma^2, \forall \bm{w}\in\mathbb{R}^M.
        \label{eq-variance}
        \end{equation}
    \item (\textbf{Degree of non-IIDness}) There exists $\kappa>0$ such that 
        \begin{equation}
            \left\| \nabla F_i(\bm{w}) - \nabla F(\bm{w}) \right\|_2 \leq \kappa, \quad \forall \bm{w} \in\mathbb{R}^M,
        \label{eq-kappa}
        \end{equation}
        where $\kappa$ measures the degree of data heterogeneity across all client nodes. When $\kappa=0$, it reduces to the IID case.
\end{itemize}
\end{assumption}

Denote $\mathbf{W}_k\triangleq [\bm{w}_k^{(i)}]_{i\in\mathcal{C}}\in\mathbb R^{M \times C}$ and $\mathbf{G}_k\triangleq [g(\bm{\xi}^{(i)}; \bm{w}_k^{(i)})]_{i\in\mathcal{C}} \in\mathbb R^{M \times C}$.
To characterize the process of model uploading and broadcasting between client nodes and edge servers, we define $\mathbf{V} \triangleq [v^{i,d}] \in \mathbb{R}^{C\times D}$ and $\mathbf{B} \triangleq [b^{d,i}] \in \mathbb R^{D\times C}$, where $v^{i,d} \triangleq \hat{m}_i \mathbbm{1}\left\{ i\in \mathcal{C}_d \right\}$ represents the data ratio of client node $i$ within the $d$-th edge cluster, and $b^{d,i} = \mathbbm{1}\left\{ i\in \mathcal{C}_d \right\}$ denotes the association between edge server $d$ and client node $i$.
Thus, the local models at client nodes evolve according to the following lemma.
\begin{lemma}\label{lm:evolve}
The local models evolve according to the following expression:
    \begin{equation}
    \mathbf{W}_{k+1} = (\mathbf{W}_k -\eta \mathbf{G}_k) \mathbf{T}_k,\, k=1,2,\dots,K,
    \label{eq:evolve}
    \end{equation}
    where the transition matrix is given by
    \begin{equation}
        \mathbf{T}_k \! = \! \left\{
        \begin{array}{ll}
        \mathbf{V}\mathbf{B},
        & \text{if}\;\mathrm{mod}\left(k,\!\tau_1\right)\!=\!0\;  % R
        \text{and}\;\mathrm{mod} \left(k,\!\tau_1\tau_2\right)\!\neq\!0, \\
        \mathbf{V}\mathbf{P}^\alpha\mathbf{B},
        & \text{if}\; \mathrm{mod}\left(k,\!\tau_1\tau_2\right)=0, \\ % Z
        \, \mathbf{I}_C, & \text{otherwise.}
        \end{array}
        \right.
        \label{T-k}
    \end{equation}
    % \vspace{-10pt}
\end{lemma}
\begin{proof}
We rewrite \eqref{eq:inter} in the matrix form as $\mathbf{\hat{Y}}_{k,l} = \mathbf{P} \mathbf{\hat{Y}}_{k,l-1}$, where $\mathbf{\hat{Y}}_{k,l} \triangleq [\bm{\hat{y}}_{k,l}^{(d)}]\in\mathbb R^{M \times D}$. According to this iterative relationship, we have $\mathbf{\hat{Y}}_{k,\alpha} = \mathbf{P}^\alpha \mathbf{\hat{Y}}_{k,0}$. Then the proof is completed by following $\mathbf{\hat{Y}}_{k,0} = \mathbf{V} \mathbf{W}_k$ and $\mathbf{W}_{k+1} = \mathbf{B}\mathbf{\hat{Y}}_{k,\alpha}$.
\end{proof}

Denote $\bm{m} \triangleq [m_i]_{i\in\mathcal{C}}$ and $\mathbf{M} \triangleq \bm{m}\mathbf{1}_C^\TT$.
We define an auxiliary global model as $\bm{u}_k \triangleq \mathbf{W}_k \bm{m}$ and the corresponding concentration matrix is $\bm{u}_k \mathbf{1}_C^\TT = \mathbf{W}_k\mathbf{M}$. 
Let $\bm{u}_k$ evolve as if $\mathbf{G}_k \mathbf{M}$ can be obtained by a centralized PS and used to update the global model in every iteration. Accordingly, we have the following relationship:
\begin{equation}
    \bm{u}_{k+1} = \bm{u}_{k} - \eta \mathbf{G}_k \bm{m}^\TT.
\label{eq:u_k}
\end{equation}
Such a relationship is desired to ensure fast convergence in terms of training iterations, as the gradients computed at all the client nodes can be leveraged to update the global model \cite{wang2019adaptive,wang2019mlsys}.
Unfortunately, this can be achieved only when edge servers receive local models and reach a consensus in each training iteration (i.e., $\tau_1=1$, $\tau_2=1$, and $\zeta^\alpha=0$).
Comparatively, the model $\mathbf{W}_k$ in SD-FEEL deviates from this desired sequence due to the existence of $\mathbf{T}_k$ in \eqref{eq:evolve}.
However, there exists an upper bound for the deviation between $\mathbf{W}_k$ and $\bm{u}_k \mathbf{1}_C^\TT$, introduced by both mini-batch sampling of SGD (i.e., $\sigma^2$) and non-IIDness across client nodes (i.e., $\kappa^2$).

\begin{lemma}\label{lemma:deviation}
With Assumption \ref{assumptions}, we have:
\begin{align}
\frac{1}{K} \sum_{k=1}^{K} \mathbb E \left[\left\| \mathbf{W}_k - \bm{u}_k \mathbf{1}_C^\TT \right\|_{\mathbf{M}}^2 \right] \label{eq:deviation}
\end{align}
\begin{align}
\leq 2\eta^2 V_1 \sigma^2 + 8\eta^2 V_2 \kappa^2 + \frac{8\eta^2 V_2}{K} \sum_{k=1}^{K} J_k, \nonumber
\end{align}
\noindent
where $V_1 \triangleq \left(\tau_1\tau_2 \frac{\zeta^{2\alpha}}{1-\zeta^{2\alpha}} + \frac{\tau_1\tau_2-1}{2} \right)/\left(1-16\eta^2L^2V_3\right)$, $V_2 \triangleq V_3/\left(1-16\eta^2L^2V_3\right)$, $V_3 \triangleq \tau_1\tau_2 \left(\tau_1\tau_2\Lambda + \frac{\tau_1\tau_2-1}{2} \frac{2-\zeta^\alpha}{1-\zeta^\alpha} \right)$, $\Lambda \!\triangleq\! \frac{\zeta^{2\alpha}}{1-\zeta^{2\alpha}} \!+\! \frac{2\zeta^\alpha}{1-\zeta^\alpha} \!+\! \frac{\zeta^{2\alpha}}{(1-\zeta^\alpha)^2}$, and $J_k \!\triangleq\! \mathbb{E} [ \| \sum_{i \in \mathcal{C}} m_i \nabla F_i(\bm{w}_k^{\!(i)}) \|_2^{2} ]$.
\end{lemma}
\begin{proof}
Since $\bm{u}_k \mathbf{1}_C^\TT = \mathbf{W}_k\mathbf{M}$, we have $\mathbf{W}_k - \bm{u}_k \mathbf{1}_C^\TT=\mathbf{W}_k(\mathbf{I}_C-\mathbf{M})$, which can be expanded as $\mathbf{W}_0(\mathbf{I}_C-\mathbf{M}) \prod_{l=1}^{k-1}\mathbf{T}_{l} - \eta \sum_{s=1}^{k-1}\mathbf{G}_{s}$ $( \prod_{l=s}^{k-1}\mathbf{T}_{l} \!-\! \mathbf{M} )$ using \eqref{eq:evolve}.
Given that $\bm{w}_{0}^{(i)}=\bm{w}_{0}, \forall i\in\mathcal{C}$, it remains to provide an upper bound for the term $\mathbb E [\| -\eta \sum_{s=1}^{k-1}\mathbf{G}_{s}( \prod_{l=s}^{k-1}\mathbf{T}_{l} - \mathbf{M} ) \|_{\mathbf{M}}^2 ]$. The proof is concluded by bounding the accumulated variance of gradients with Assumption \ref{assumptions} and the Jensen's inequality.
The complete proof is referred to Appendix C in \cite{sun2021semi}.
\end{proof}

As aforementioned, the training objective is to output a global model $\bm{u}_k$ which minimizes the global objective function. 
We now characterize how the expected loss of $\bm{u}_k$ changes in two consecutive iterations in the following lemma.
\begin{lemma}\label{lemma-one-step}
With Assumption \ref{assumptions}, the expected change of the local loss functions in two consecutive iterations is bounded as follows:
\begin{equation}
\begin{split}
    &\quad\; \mathbb{E}[F(\bm{u}_{k+1})] - \mathbb{E} [F(\bm{u}_{k})] \\
    & \leq - \frac{\eta}{2} \mathbb{E} \left[\left\| \nabla F(\bm{u}_k) \right\|_2^2 \right]
    + \frac{\eta^2L}{2} \sum_{i\in\mathcal{C}} m_i^2 \sigma^2 \\
    & -\left(\frac{\eta}{2}-\frac{\eta^2 L}{2}\right) J_k
    +\frac{\eta L^2}{2} \mathbb{E} \left[\left\| \mathbf{W}_k(\mathbf{I}_C-\mathbf{M})\right\|_{\mathbf{M}}^2\right].
\end{split}
\label{eq:one-step}
\end{equation}
\end{lemma}
\begin{proof}
Following Lemma 8 in \cite{d2}, the proof is completed by plugging the right-hand side (RHS) of \eqref{eq:u_k} into the first-order Taylor expansion of $\nabla F(\bm{u}_{k+1})$ and leveraging Assumption \ref{assumptions}.
Please refer to the detailed proof in \cite{sun2021semi}.
\end{proof}

With the above lemmas, we prove the convergence of Algorithm \ref{alg-1} in the following theorem.
\begin{theorem}\label{thm-1}
    If the learning rate $\eta$ satisfies:
    \begin{equation}
        1-\eta L-8\eta^2L^2 V_2\geq 0, 1-16\eta^2L^2V_3 > 0,
        \setlength{\belowdisplayskip}{-3pt}
        \label{eq:lr}
    \end{equation}
    we have:
    \begin{equation}
    % \begin{split}
    \frac{1}{K} \sum_{k=1}^{K} \mathbb{E} \left[\left\| \nabla F(\bm{u}_k) \right\|_2^2\right]
    \leq \frac{2\Delta}{\eta K}
    + \eta L \Phi_0 + \eta^2 L^2 \Phi(\tau_1,\tau_2,\alpha,\zeta),
    % \end{split}
    \label{eq:thm}
    \end{equation}
    where $\Delta \triangleq \mathbb{E}\left[F(\bm{u}_1)\right]- F(\bm{u}^*)$, $\bm{u}^*\triangleq \arg\min_{\bm{w}}F(\bm{w})$, $\Phi_0 \triangleq \sum_{i\in\mathcal{C}} m_i^2 \sigma^2$, and $\Phi(\tau_1,\tau_2,\alpha,\zeta) \triangleq 2 V_1\sigma^2 + 8 V_2 \kappa^2$.
\end{theorem}
\begin{proof}
We sum up both sides of \eqref{eq:one-step} over $k=1,2,\dots,K$, and divide them by $K$. Then we apply \eqref{eq:deviation} to its RHS and choose a suitable $\eta$ that satisfies the conditions in \eqref{eq:lr} to eliminate $J_k$. The proof is concluded by rearranging terms.
\end{proof}

\subsection{Discussions}
The result of Theorem \ref{thm-1} provides us with various insights, as presented in this subsection.
%%%%%%%%% corollary
\begin{corollary}\label{corollary-1}
\textbf{(Convergence rate)}
    If the learning rate is chosen as $\eta=\mathcal{O}\left(\frac{1}{L\sqrt{K}}\right)$, \eqref{eq:thm} can be simplified as follows:
    \begin{equation}
        \frac{1}{K} \sum_{k=1}^{K} \mathbb{E} \left[ \left\| \nabla F(\bm{u}_k) \right\|_2^2 \right]
        \!\leq\! \mathcal{O}\left(\frac{2L\Delta+\Phi_0}{\sqrt{K}} \right)
        + \mathcal{O}\left( \frac{\Phi(\tau_1,\tau_2,\alpha,\zeta)}{K} \right),
    \label{eq:coro}
    \end{equation}
    which implies that Algorithm \ref{alg-1} converges within $\mathcal{O}\left(\frac{1}{\epsilon^2}\right)$ training iterations to find an $\epsilon$-approximate solution (i.e., $\frac{1}{K} \sum_{k=1}^{K}\mathbb{E} \left[ \left\| \nabla F(\bm{u}_k) \right\|_2^2 \right] \leq \epsilon$).
\end{corollary}

The results of Theorem \ref{thm-1} and Corollary \ref{corollary-1} contain two variance terms, i.e., $\Phi_0$  and $\Phi(\tau_1,\tau_2,\alpha,\zeta)$.
The constant term $\Phi_0$ is the same as the bound in previous literature of centralized SGD \cite{bottou2018optimization} and FL frameworks \cite{d2,zhangweiyi}.
However, infrequent aggregations in SD-FEEL result in model divergence across client nodes, reflected through $\Phi(\tau_1,\tau_2,\alpha,\zeta)$.

%%%%%%%% tau_1/tau_2
\begin{remark}\label{rm:tau}
\textbf{(Effect of aggregation periods)}
By taking the first-order derivative, we see that the term $\Phi(\tau_1,\tau_2,\alpha,\zeta)$ increases with both $\tau_1$ and $\tau_2$.
With more local iterations performed, the models on client nodes become more biased towards local datasets, which slows down the convergence speed \cite{wang2019mlsys,wang2020tackling}.
Therefore, to minimize the error floor in the RHS of \eqref{eq:thm}, we prefer setting $\tau_1=1$ and $\tau_2=1$.
Nevertheless, as will be seen in Section \ref{sec:experiment}, with smaller values of $\tau_1$ and $\tau_2$, such frequent aggregations incur huge latency in both intra-cluster and inter-cluster communications. 
Thus, $\tau_1$ and $\tau_2$ should be properly selected to achieve fast convergence with respect to the wall-clock time.
\end{remark}

%%%%%%%% alpha, connectivity
\begin{remark}\label{rm:topo}
\textbf{(Effect of network among edge servers)}
The RHS of \eqref{eq:thm} also increases with $\zeta$ and $\alpha$, 
Note that when $\zeta^\alpha \!\rightarrow\! 0$ (i.e., with the fully connected topology or $\alpha \!\rightarrow\! \infty$), the edge servers can reach perfect consensus (i.e., $\bm{y}_{k}^{(d)} \!\leftarrow\! \sum_{j\in\mathcal{D}} \Tilde{m}_j \bm{y}_{k}^{(j)}$) after inter-cluster model aggregation.
\begin{itemize}
    \item A smaller value of $\zeta$ (i.e., a more connected graph among the edge servers) results in faster convergence. Fig. \ref{fig:three-topo} shows several typical topologies formed by six edge servers, including the star ($\zeta=0.71$), ring ($\zeta=0.6$), partially connected ($\zeta=0.33$), and fully connected ($\zeta=0$) topologies. It is clear that the fully-connected network topology can achieve the best performance with a fixed $\alpha$.
    % NOTE: zeta = \lambda(P), not \lambda(graph)
    \item Increasing $\alpha$ (i.e., raising the inter-server communication overhead) can reduce the variance terms in the RHS of \eqref{eq:thm} and thus speed up the convergence. Nevertheless, such benefits diminish as $\alpha$ increases. As will be demonstrated in Section \ref{sec:experiment}, there is no additional gain after $\alpha$ is beyond a certain value.
\end{itemize}
\end{remark}

\begin{figure*}[t]
    \centering
    \setlength\abovecaptionskip{-1pt}
    \setlength\belowcaptionskip{-10pt}
    \includegraphics[width=0.8\textwidth]{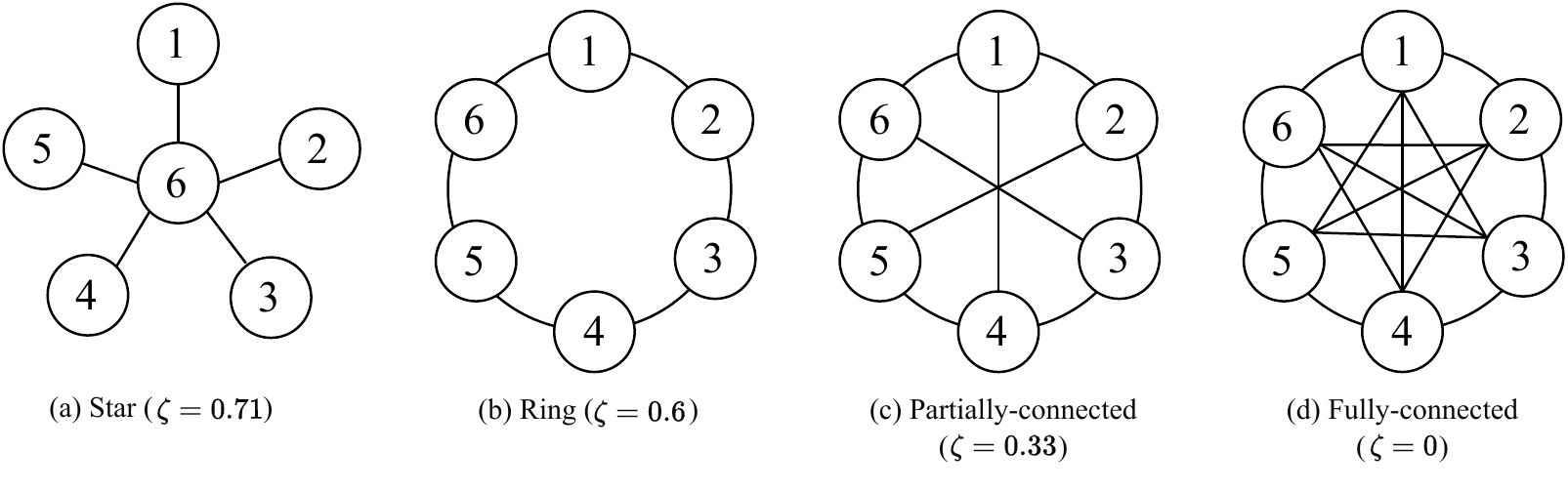}
    \caption{Typical network topologies of the edge servers ($\zeta$ denotes the second largest eigenvalue of matrix $\mathbf{P}$ defined in \eqref{eq5}).}
    \label{fig:three-topo}
\end{figure*}

%%%%%%%% reduce to hier, fl, centralized
\begin{remark}\label{rm:reduce} 
\textbf{(Comparison with HierFAVG)}
If edge servers can achieve a perfect consensus in each training iteration, i.e., $\zeta^\alpha=0$, our results in Theorem \ref{thm-1} and Corollary \ref{corollary-1} reduce to the case of HierFAVG \cite{HierFL}.
Accordingly, Corollary \ref{corollary-1} indicates that HierFAVG requires fewer iterations to converge. Nevertheless, SD-FEEL adopts inter-server communication to replace the communication between the edge server and the Cloud, which significantly reduces the per-iteration latency. 
As a result, the total training time of the two architectures depends on the application scenarios. Here we discuss two special cases as follows:
\begin{itemize}
    \item When the network among edge servers is fully-connected, SD-FEEL always leads to a better model than HierFAVG within a given training time.
    \item When the inter-server communication latency is adequately low, the edge servers are allowed to perform sufficient rounds of model sharing to reach consensus. Thus, SD-FEEL converges faster than HierFAVG with respect to the wall-clock time.
\end{itemize}
\end{remark}

When client nodes differ slightly in computation resources, i.e., with a small $H$, SD-FEEL converges fast with Algorithm \ref{alg-1}. Nevertheless, the training efficiency may be degraded if there is a huge disparity in computational speeds, i.e., $H$ is large. In the next section, we propose an asynchronous training algorithm to alleviate such an issue.

%%%%%%%%%%%%%%%%%%%%%%%%%%%%%%%%%%%%%%%%%%
%%%%%%%%%%%%%%  Extension  %%%%%%%%%%%%%%
%%%%%%%%%%%%%%%%%%%%%%%%%%%%%%%%%%%%%%%%%%
\section{Asynchronous Training for SD-FEEL}\label{sec:async}
When client nodes have significantly diverse computational speeds (i.e., $H \gg 1$), performing an equal number of local iterations results in an idle time of the fast client nodes since they have to wait for the straggling ones to complete local training. To mitigate such an issue, we propose an asynchronous training algorithm for SD-FEEL where each edge server presets a deadline for local training and proceeds to the next training iteration once it completes inter-cluster model aggregation of the current iteration.

\subsection{Training Algorithm}\label{sec:async-alg}
Denote the training iteration counter by $t$, which gives the total number of training iterations completed by edge clusters. 
Before entering the training process, each edge server presets a deadline for local model updates given the computational resources of the coordinated client nodes. Specifically, edge server $d$ sets the value of $T_\text{comp}^{(d)}$ for the client nodes in $\mathcal{C}_d$.
The determination of $T_\text{comp}^{(d)}$ is beyond the scope of this paper, but in general, it follows the same principles mentioned in Remark \ref{rm:tau}, namely to ensure effective training as well as avoid large model inconsistency in model aggregations.
An illustration of the asynchronous training process is shown in Fig. \ref{fig:async-2cluster}. We present the details of three key stages in each iteration as follows.
% \end{figure}

\subsubsection{Local Model Update}
In the $d$-th edge cluster, client node $i$ performs $\theta_i= h_i \beta \in \{\theta_{\text{min}},\theta_{\text{min}}+1,\dots,\theta_{\text{max}}\}$ local epochs using mini-batch SGD within the duration of $T_\text{comp}^{(d)}$ and $\beta$ is related to the complexity of training task and the batch size.
Denote the model of client node $i$ at the beginning of the $l$-th local training epoch in the $t$-th global iteration by $\bm{w}_{t,l}^{(i)}$. Then we have the following expression:
\begin{equation}
\bm{w}_{t,l+1}^{(i)} \!\!\leftarrow\! \bm{w}_{t,l}^{(i)} \!-\! \eta g(\bm{\xi}_{t,l}^{(i)}; \bm{w}_{t,l}^{(i)}), l \!\in \!\! \left\{0,1,\!\dots\!,\theta_i\!-\!1\right\}, i \!\in\! \mathcal{C}.
\label{eq:update-async}
\end{equation}
Given that the varying numbers of local epochs among client nodes may incur model bias \cite{wang2020tackling}, client node $i$ normalizes the local updates by $\theta_i$ as follows:
\begin{equation}
    \bm{\Delta}_{t}^{(i)} \!\triangleq\! \frac{1}{\theta_i} \! \left(\bm{w}_{t,\theta_i}^{(i)} - \bm{w}_{t,0}^{(i)}\right) \!=\! - \frac{\eta}{\theta_i} \! \sum_{l=0}^{\theta_i-1} \! g(\bm{\xi}_{t,l}^{(i)}; \bm{w}_{t,l}^{(i)}), i\in \mathcal{C}.
\label{eq:delta-async}
\end{equation}

\subsubsection{Intra-cluster Model Aggregation}
Once the deadline $T_\text{comp}^{(d)}$ arrives, edge server $d$ collects the normalized model updates from its associated client nodes $\mathcal{C}_d$ and computes the weighted average with the weighting factors $\{\hat{m}_i\}$'s.
Denote the most updated model at edge server $d$ at the beginning of the $t$-th iteration by $\bm{y}_t^{(d)}$.
Then such an update is added to $\bm{y}_t^{(d)}$ as the implementation of gradient descent, which can be expressed as follows:
% \vspace{-2pt}
\begin{equation}
    \bm{\hat{y}}_{t}^{(d)} 
    \leftarrow \bm{y}_t^{(d)} + \overline{\theta}_d \sum_{i\in\mathcal{C}_d} \hat{m}_i \bm{\Delta}_t^{(i)}, d \in \mathcal{D},
\label{eq:intra-async}
\end{equation}
where $\overline{\theta}_d \triangleq \sum_{i\in \mathcal{C}_d} \hat{m}_i \theta_i$ is the weighted average of the numbers of local epochs completed by the client nodes.

\subsubsection{Inter-cluster Model Aggregation}
After intra-cluster model aggregation, edge server $d$ exchanges the local model $\bm{\hat{y}}_{t}^{(d)}$ with its neighboring edge servers in $\mathcal{N}_d$. The models maintained by these edge servers are accordingly updated as follows:
\begin{equation}
    \bm{y}_{t}^{(j)} \leftarrow \sum_{j^\prime \in \mathcal{N}_{j}\cup\{j\}} p_{t}^{j^\prime,j} \bm{\hat{y}}_{t}^{(j^\prime)}, j \in \mathcal{N}_{d}\cup\{d\},
    \label{eq:inter-async}
    % \vspace{-5pt}
\end{equation}
where $\mathbf{P}_t \triangleq \{p_t^{j^\prime,j}\}\in\mathbb{R}^{D \times D}$ denotes the mixing matrix in the $t$-th iteration. Note that the neighboring edge server $j$ maintains a model updated in a previous global iteration $t^\prime(j)<t$, the gap of which is named the \emph{iteration gap}, i.e., $\delta_t^{(j)} \triangleq t - t^\prime(j)$. Since a larger value of $\delta_t^{(j)}$ implies that the model is staler and has less value to other edge clusters \cite{xie2019asynchronous}, we design a staleness-aware mixing matrix as follows:
\begin{equation}
    p_{t}^{i,j} = \left\{
        \begin{array}{ll}
        \frac{\psi(\delta_t^{(j)})}{\Psi_t^{(j)}}, & \text{if} \; j=d \; \text{and} \; i\in\mathcal{N}_{d}\cup\{d\}, \\
        p_{t}^{j,i}, \; & \text{if} \; j \in\mathcal{N}_{d} \; \text{and} \; i=d,\\
        1-p_{t}^{d,j}, \; & \text{if} \; j \; \in\mathcal{N}_{d} \; \text{and} \; i=j ,\\
        1, \; & \text{if} \; j \; \notin \mathcal{N}_{d}\cup\{d\} \; \text{and} \; i=j ,\\
        0, & \text{otherwise,}
        \end{array}
        \right.
\label{eq:pk}
\end{equation}
where $\psi(x)$ is a general non-increasing function of $x$ and $\Psi_t^{(j)} \triangleq \sum_{i\in\mathcal{N}_{j}\cup\{j\}} \psi(\delta_t^{(i)})$. Then the model $\bm{y}_t^{(d)}$ is broadcasted to the client nodes in $\mathcal{C}_d$ according to $\bm{w}_{t+1,0}^{(i)} \leftarrow \bm{y}_{t}^{(d)}, i \in \mathcal{C}_d$. For illustration, consider an example with three edge clusters $d\in\{1,2,3\}$ arranged in sequential order. When edge cluster $1$ triggers the inter-cluster model aggregation in the $t$-th training iteration, the iteration gaps of its neighboring edge cluster ($\mathcal{N}_1=\{2\}$) is $\delta_t^{(2)}=2$. The corresponding mixing matrix is $\mathbf{P}_t = \left[\frac{\psi(0)}{\Psi_t^{(1)}}, \frac{\psi(2)}{\Psi_t^{(1)}}, 0; \frac{\psi(2)}{\Psi_t^{(1)}}, 1-\frac{\psi(2)}{\Psi_t^{(1)}}, 0; 0,0,1 \right]$.

The above steps repeat until timeout or the values of local loss at all the client nodes cannot be further reduced. Assume $T$ is the global iteration index at that time. Then the system enters the consensus phase to output a global model $\sum_{d\in\mathcal{D}} \Tilde{m}_d \bm{y}_{T}^{(d)}$.

\subsection{Convergence Analysis}

We define an auxiliary global model at the $t$-th training iteration as $\overline{\bm{y}}_{t} \triangleq \sum_{d\in\mathcal{D}} \Tilde{m}_d \bm{y}_{t}^{(d)}$, the evolution of which is expressed as:
\begin{equation}
    \bm{\overline{y}}_{t+1} = \bm{\overline{y}}_{t} - \eta \mathbf{\hat{G}}_t \bm{\Tilde{m}}^\TT \mathbf{\Lambda},
\label{eq:y_t}
\end{equation}
where $\mathbf{\hat{G}}_t \triangleq [ \sum_{i\in\mathcal{C}_d} \frac{\hat{m}_i}{\theta_i} \sum_{l=0}^{\theta_i-1}$ $g(\bm{\xi}_{t,l}^{(i)};\bm{w}_{t,l}^{(i)}) ]_{d\in\mathcal{D}} \!\in\! \mathbb{R}^{M \times D}$, $\bm{\Tilde{m}}\triangleq[\Tilde{m}_d]_{d\in\mathcal{D}}$, and $\mathbf{\Lambda} \triangleq \text{diag}(\overline{\theta}_1, \overline{\theta}_2,\dots,$ $\overline{\theta}_D)$.
Since the client nodes perform different local epochs, we prove the convergence by focusing on the servers' models, in contrast to the client models as in the proof of Section \ref{sec:theory}. The equivalence $\overline{\bm{y}}_{t}=\sum_{i\in\mathcal{C}} m_i \bm{w}_{t+1,0}^{(i)}$ always holds if all edge servers broadcast the models to the associated client nodes.

The analysis follows the procedures in subsection \ref{sec:convergence}, i.e., using accumulated gradients to provide an upper bound for model deviation.
Nevertheless, the asynchronous training incurs additional model inconsistency among edge clusters. 
Specifically, iteration counter $t$ increases once an edge cluster completes a training iteration.
However, the client nodes in other edge clusters are training on stale models as aforementioned.
To characterize this phenomenon, we define $\bm{a}_{\Tilde{t}}^{(d)} \triangleq \bm{a}_{t-\delta_t^{(d)}}^{(d)}$ (respectively $\bm{a}_{\Tilde{t}}^{(i)} \triangleq \bm{a}_{t-\delta_t^{(d)}}^{(i)}, i \!\in\! \mathcal{C}_d$) as the delayed model or gradient at the $d$-th edge server (respectively the $i$-th client node) in the $t$-th iteration.
The following lemma shows that $\delta_t^{(d)}$ is upper bounded throughout the whole training process.
\begin{lemma}\label{lemma-delta}
\textbf{(Bounded iteration gap)}
There exists a positive integer $\delta_\text{max}$ such that $\delta_t^{(d)} \leq \delta_\text{max}, \forall t \in \mathbb{N}, d\in\mathcal{D}$.
\end{lemma}
\begin{proof}
After setting $\left\{ T_\text{comp}^{(d)} \right\}$'s, the training latency for each iteration is fixed as $T_\text{iter}^{(d)}$.
During one training iteration of the slowest edge cluster $j^*$, $\delta_\text{max}= \sum_{d\in\mathcal{D}} \left( \left\lceil \frac{T_\text{iter}^{(j^*)}}{T_\text{iter}^{(d)}} \right\rceil -1 \right)$ gives a maximal value for the number of total iterations that other edge clusters have completed.
\end{proof}

To show the convergence, we begin with upper bounding the expected change of loss functions in two consecutive iterations in Lemma \ref{lemma-5}.
\begin{lemma}\label{lemma-5}
The expected change of the global loss function in two consecutive iterations is bounded as follows:

\begin{align}
    &\quad \mathbb{E}[F(\bm{\overline{y}}_{t+1})] - \mathbb{E} [F(\bm{\overline{y}}_{t})] \label{eq:loss-change-async} \\
    & \leq - \frac{1}{2} \eta\theta_{\mathrm{min}} \mathbb{E} \left[\left\| \nabla F(\bm{\overline{y}}_t) \right\|^2 \right]
    + \frac{1}{2} \eta^2 L \theta_{\mathrm{max}}^2 \theta_{\mathrm{min}}^{-1} \sum_{i\in\mathcal{C}} m_i^2 \sigma^2 \nonumber \\
    & - \frac{\eta}{2} (\theta_{\mathrm{min}} - \eta L  \theta_{\mathrm{max}}^2) Q_{\Tilde{t}}
    +\frac{1}{2} \eta \theta_{\mathrm{min}} \underbrace{\mathbb{E} \left[ \left\| \nabla F(\bm{\overline{y}}_t) -  \nabla\mathbf{\hat{F}}_{\Tilde{t}} \bm{\Tilde{m}}^\TT \right\|^2 \right]}_{\mathcal{E}_t}, \nonumber
\end{align}
where $Q_{\Tilde{t}} \!\triangleq\! \mathbb{E} [ \| \nabla\mathbf{\hat{F}}_{\Tilde{t}} \bm{\Tilde{m}}^\TT \|^2 ]$ and $\nabla\mathbf{\hat{F}}_{\Tilde{t}} \!\triangleq\! [ \sum_{i\in\mathcal{C}_d} \! \frac{\hat{m}}{\theta_i} \! \sum_{l=0}^{\theta_i-1} \nabla F_i (\bm{w}_{\Tilde{t},l}^{(i)}) ]_{d\in\mathcal{D}}$.
\end{lemma}
\begin{proof}
Similar to the proof of Lemma \ref{lemma-one-step}, by plugging the RHS of \eqref{eq:y_t} into the first-order Taylor expansion of $\nabla F(\bm{\overline{y}}_{t+1})$ and following Lemma 8 in \cite{d2}, we conclude the proof.
\end{proof}

The term $\mathcal{E}_t$ in \eqref{eq:loss-change-async} measures the degree to which the gradients collected from client nodes (i.e., $\nabla\mathbf{\hat{F}}_{\Tilde{t}} \bm{\Tilde{m}}^\TT$) deviate from the desired gradient of global model (i.e., $\nabla F(\bm{\overline{y}}_t)$).
Using the $L$-smoothness and the Jensen's inequality (i.e., $\|\bm{a}+\bm{b} \|^2 \leq 2\| \bm{a} \|^2 + 2\| \bm{b} \|^2, \forall \bm{a},\bm{b} \in\mathbb{R}^d$), we derive an upper bound for the term $\mathcal{E}_t$, which can be decomposed as follows:
\begin{equation}
\begin{split}
    \mathcal{E}_t &\leq 2 L^2\! \underbrace{\mathbb{E} \left[\left\| \bm{\overline{y}}_t - \bm{\overline{y}}_{\Tilde{t}} \right\|^2\right]}_{\mathcal{E}_{t,1}} 
    + 4L^2 \! \underbrace{\sum_{d\in \mathcal{D}} \Tilde{m}_d \mathbb{E}\left[ \left\| \bm{\overline{y}}_{\Tilde{t}} - \bm{y}_{\Tilde{t}}^{(d)} \right\|^2 \right]}_{\mathcal{E}_{t,2}} \\
    &+ 4L^2 \! \underbrace{ \sum_{d\in \mathcal{D}} \!\Tilde{m}_d \sum_{i\in \mathcal{C}_d}\! \frac{\hat{m}_i}{\theta_i} \sum_{l=0}^{\theta_i-1} \! \mathbb{E} \!\left[ \left\| \bm{y}_{\Tilde{t}}^{(d)} - \bm{w}_{\Tilde{t},l}^{(i)} \right\|^2 \right]}_{\mathcal{E}_{t,3}}\!.
\end{split}
\label{eq:respective}
\end{equation}
In the RHS of \eqref{eq:respective}, the term $\mathcal{E}_{t,1}$ quantifies the influence of the staleness, which can be upper bounded by using Lemma \ref{lemma-delta}.
The term $\mathcal{E}_{t,2}$ measures the inter-cluster model divergence between an edge server and the averaged model $\bm{\overline{y}}_{\Tilde{t}}$, while the term $\mathcal{E}_{t,3}$ measures the weighted sum of the model divergence between edge clusters.
After respectively bounding three terms, we obtain the following lemma.

\begin{lemma}\label{lemma-e-t}
With Assumption \ref{assumptions}, we have:
\begin{equation}
\begin{split}
    \frac{1}{T} \sum_{t=0}^{T-1} \mathcal{E}_t
    & \leq A (\theta_\mathrm{min},\theta_\mathrm{max},\delta_\mathrm{max}) \sigma^2 \\
    & + B (\theta_\mathrm{min},\theta_\mathrm{max},\delta_\mathrm{max}) \kappa^2 
    + C (\theta_\mathrm{max},\delta_\mathrm{max}) Q_{\Tilde{t}},
    % \vspace{-10pt}
\end{split}
\label{eq-lemma}
\end{equation}
where $A (\theta_\mathrm{min}, \theta_\mathrm{max},\delta_\mathrm{max}) \triangleq 4 \eta^2 L^2 \delta_\mathrm{max}^2 \theta_{\mathrm{max}}^2 \theta_\mathrm{min}^{-1} U_4
    + \frac{4\eta^2L^2 (\theta_\mathrm{max}-1)}{1-2\eta^2L^2 U_2}
    + 8\eta^2 L^2 \theta_{\mathrm{max}}^2 \theta_\mathrm{min}^{-1} \frac{U_3}{T} \!\sum\limits_{t=0}^{T-1} \sum\limits_{s=1}^{t-1} \!\rho_{s,t-1}^2$,
    $B (\theta_\mathrm{min},\theta_\mathrm{max},\delta_\mathrm{max}) \triangleq 8 \eta^2 L^2 \delta_\mathrm{max}^2 \theta_{\mathrm{max}}^2 \theta_\mathrm{min}^{-1} U_4 + \frac{24\eta^2L^2U_2}{1 \!-\! 2\eta^2L^2 U_2}
    + 16 \eta^2 L^2 \theta_{\mathrm{max}}^2 \theta_\mathrm{min}^{-1} U_3 \frac{1}{T} \!\sum_{t=0}^{T-1} ( \sum_{s=1}^{t-1} \rho_{s,t-1} )^2$,
    $ C (\theta_\mathrm{max},\delta_\mathrm{max}) \triangleq 8 \eta^2 L^2 \delta_\mathrm{max}^2 \theta_\mathrm{max} U_4
    + 16 \eta^2 L^2 \theta_\mathrm{max}^2 U_3 \frac{1}{T} \!\sum_{t=s+1}^{T-1} \rho_{s,t-1} ( \sum_{l=1}^{t-1} \rho_{l,t-1} )$,
    $U_2 \!\triangleq\! \theta_\mathrm{max} (\theta_\mathrm{max}\!-\!1)$, 
    $U_3 \!\triangleq\! \frac{1\!+\!4\eta^2L^2 U_2}{1\!-\!2\eta^2L^2 U_2}$,
    $U_4 \!\triangleq\! \frac{1 \!+\! 22\eta^2L^2 U_2}{1\!-\!2\eta^2L^2 U_2}$,
    and $\rho_{s,t-1} \triangleq \| \prod_{l=s}^{t-1} \mathbf{P}_{l} - \mathbf{M} \|_\mathrm{op}$.
\end{lemma}
\begin{proof}
The proof is obtained by respectively bounding the three terms in the RHS of \eqref{eq:respective}.
Please refer to the Appendix.
\end{proof}

We are now ready to prove the convergence of the asynchronous training algorithm.
\begin{theorem} \label{thm-2}
With Assumption \ref{assumptions}, if the learning rate $\eta$ satisfies
    \begin{equation}
    1 - \eta L  \theta_\mathrm{max}^2 \theta_\mathrm{min}^{-1} - C  (\theta_\mathrm{max},\delta_\mathrm{max}) \geq 0, 1 - 2\eta^2L^2 U_2 > 0,
    \label{eq:lr-async}
    \end{equation}
    we have
    \begin{equation}
    \begin{split}
    &\quad \frac{1}{T} \sum_{t=0}^{T-1} \mathbb{E} \left[\left\| \nabla F(\bm{\overline{y}}_t) \right\|^2 \right] \\
    &\leq\frac{2 \{\mathbb{E}[F(\bm{y}_{0})] - \mathbb{E} [F(\bm{\overline{y}}_{T})]\}}{\eta \theta_\mathrm{min} U_1 T}
    + \frac{\eta L \theta_\mathrm{max}^2}{U_1 \theta_\mathrm{min}^2} \sum_{i\in\mathcal{C}} m_i^2 \sigma^2 \\
    % \vspace{-5pt}
    &+ A (\theta_\mathrm{min}, \theta_\mathrm{max},\delta_\mathrm{max}) \frac{\sigma^2}{U_1} + B (\theta_\mathrm{min}, \theta_\mathrm{max},\delta_\mathrm{max}) \frac{\kappa^2}{U_1},
    \end{split}
    \label{eq:thm-2}
    \end{equation}
    where $U_1 \!\triangleq\! \frac{1-14\eta^2L^2 U_2}{1-2\eta^2L^2 U_2}$.
\end{theorem}
\begin{proof} 
By summing up both sides of \eqref{eq:loss-change-async}, plugging \eqref{eq-lemma} into the RHS, and then choosing the learning rate as \eqref{eq:lr-async} to eliminate $Q_{\Tilde{t}}$, we conclude the proof.
\end{proof}

\begin{remark}
\textbf{(Convergence rate)}
If we choose the learning rate as $\eta \!=\! \mathcal{O} \left( \! \frac{1}{L\sqrt{T}} \! \right)$, the RHS of \eqref{eq:thm-2} decreases at a speed of $\mathcal{O}\left(\frac{1}{\sqrt{T}}\right)$ and approaches zero when $T \!\! \rightarrow \!\! \infty$, which ensures convergence.
\end{remark}

%%%%%%%%%%%%%%%%%%%%%%%%%%%%%%%%%%%%%%%%%%
%%%%%%%%% Simulation Evaluation %%%%%%%%%
%%%%%%%%%%%%%%%%%%%%%%%%%%%%%%%%%%%%%%%%%%
\section{Simulations}\label{sec:experiment}

\subsection{Settings} 
We simulate an SD-FEEL system with 50 client nodes and 10 edge servers. Without otherwise specified, it is assumed that each edge server coordinates five client nodes.
We consider a ring topology of edge servers unless otherwise specified.
We evaluate SD-FEEL on two benchmark datasets for image classification, i.e., the MNIST \cite{mnist} and CIFAR-10 \cite{cifar10} datasets, each of which has ten classes of labels. 
For the non-IID setting, we adopt the skewed label partition \cite{hsieh2020non} for the MNIST dataset, where each client node has $c$ (with a default value of 2) random classes of data samples.
For the CIFAR-10 dataset, we utilize a Dirichlet distribution $\text{Dir}_{50}(\beta)$ to sample the probabilities $\{p^\prime_{l,i}\}$'s, which is the proportion of the training samples of class $l$ assigned to the $i$-th client node \cite{yurochkin2019bayesian}.
Here $\beta$ (with a default value of 0.5) is the concentration parameter and a smaller value of $\beta$ results in a more uneven local distribution across the client nodes.
Following \cite{HierFL}, we train a convolutional neural network (CNN) with two $5\times 5$ convolutional layers and $M=21,840$ trainable parameters on the MNIST dataset, and another CNN with six convolutional layers that consists of $M=5,852,170$ trainable parameters on the CIFAR-10 dataset.
The mini-batch SGD is employed with a batch size of 10, and the learning rate is set as 0.001 and 0.01 for the MNIST and CIFAR-10 datasets, respectively. 

To demonstrate the effectiveness of SD-FEEL, we adopt three FL schemes as baselines, including 1) FedAvg \cite{mcmahan2017communication}, 2) HierFAVG \cite{HierFL}, and 3) FEEL \cite{lim2020federated}.
We implement the three baseline schemes based on FedML \cite{he2020fedml}, which is an open-source research library for FL.
Note that FEEL is an edge-assisted FL scheme with a single edge server randomly scheduling five client nodes in each iteration. 
For fair comparisons, each edge server in HierFAVG, FEEL, and SD-FEEL is assumed to have an equal number of orthogonal uplink wireless channels.

\subsection{Training Latency} 
For synchronous SD-FEEL, latency of $K$ training iterations can be calculated as $T_\text{tot} = K \left( T_\text{comp}^\text{ct} + \frac{1}{\tau_1} T_\text{comm}^\text{ct-sr} + \frac{\alpha}{\tau_1\tau_2} T_\text{comm}^\text{sr-sr} \right)$,
where $T_\text{comp}^\text{ct}$ is the computation latency for each local iteration at client nodes, $T_\text{comm}^\text{ct-sr}$ denotes the model uploading latency from a client node to its associated edge server, and $T_\text{comm}^\text{sr-sr}$ is the model transmission latency between neighboring edge servers. The training latency of other baseline FL schemes can be calculated similarly. Following \cite{smith2017federated}, the averaged computation time is assumed to be $T_\text{comp}^\text{ct} \!=\! \frac{N_\text{MAC}}{C_\text{CPU}}$, where $N_\text{MAC}$ is the number of the floating-point operations (FLOPs) for one local iteration, and $C_\text{CPU}=10 \,\text{GFLOPS}$ denotes the CPU's computing bandwidth for the slowest device. The numbers of FLOPs required for local training at each iteration are $N_\text{MAC} \!=\! 487.54 \,\text{KFLOPs}$ for the MNIST dataset and $N_\text{MAC} \!=\! 138.4 \,\text{MFLOPs}$ for the CIFAR-10 dataset\footnote{These values are calculated using OpCounter, which is an open-source model analysis library available at \url{https://github.com/Lyken17/pytorch-OpCounter}.}.
The communication latency is expressed as $T_\text{comm} \!=\! \frac{M_{\rm{bit}}}{R}$ with $M_{\rm{bit}} \!=\! 32\,\text{Mbits}$ and $R$ denoting the transmission rate. Specifically, we assume that the client nodes communicate with the associated edge servers using orthogonal channels and there is no inter-cluster interference \cite{shi2020joint}. The transmission rate is assumed to be $R^\text{ct-sr}\!=\!B\log_2{(1+\text{SNR}}) \!\approx\! 5\,\text{Mbps}$, where $B\!=\!1\,\text{MHz}$ and $\text{SNR}\!=\!15\,\text{dB}$. The edge server communicates with the neighboring edge servers via high-speed links with the bandwidth of $50\,\text{Mbps}$ \cite{hu2020coedge}, and the bandwidth from the edge servers to the Cloud is set as $5\,\text{Mbps}$. Accordingly, the transmission rate from the client nodes to the Cloud is given by $R^\text{ct-cd}\!=\!2.5\,\text{Mbps}$.

\subsection{Results}\label{sec:result}

\subsubsection{Convergence Speed and Generalization Performance}
We show the training loss with respect to the training time for both the MNIST and CIFAR-10 datasets in Fig. \ref{fig:convergence}. 
We see that the training loss of SD-FEEL drops rapidly in the initial training stage, and SD-FEEL converges at around $40$ seconds and $250$ minutes for the MNIST and CIFAR-10 datasets, respectively. Comparatively, HierFAVG and FedAvg fall behind due to the high communication latency with the Cloud-based PS.
Besides, FEEL undergoes an unstable and slower training process as the edge server has access to a limited number of client nodes and thus leverages much less training data.
Fig. \ref{fig:accuracy} shows the generalization performance in terms of the test accuracy over time. Within the given training time, the test accuracy of SD-FEEL improves over FedAvg and FEEL due to efficient communication across edge servers, while their learned models are still unusable.
In addition, we notice that in Fig. \ref{fig:convergence}(a) and Fig. \ref{fig:accuracy}(a), SD-FEEL and HierFAVG have a small gap in terms of training loss and test accuracy, respectively, since the computation latency dominates the training time for this task on the MNIST dataset.
In Fig. \ref{fig:accuracy}, we also show the required training latency of different FL schemes to meet 90\% (respectively 60\%) in test accuracy on the MNIST (respectively CIFAR-10) dataset.
It is observed that SD-FEEL saves substantial training time for the given training task.

\begin{figure}[!t]
\centering
\subfigure{
\begin{minipage}[t]{0.5\linewidth}
\centering
\includegraphics[width=1\linewidth]{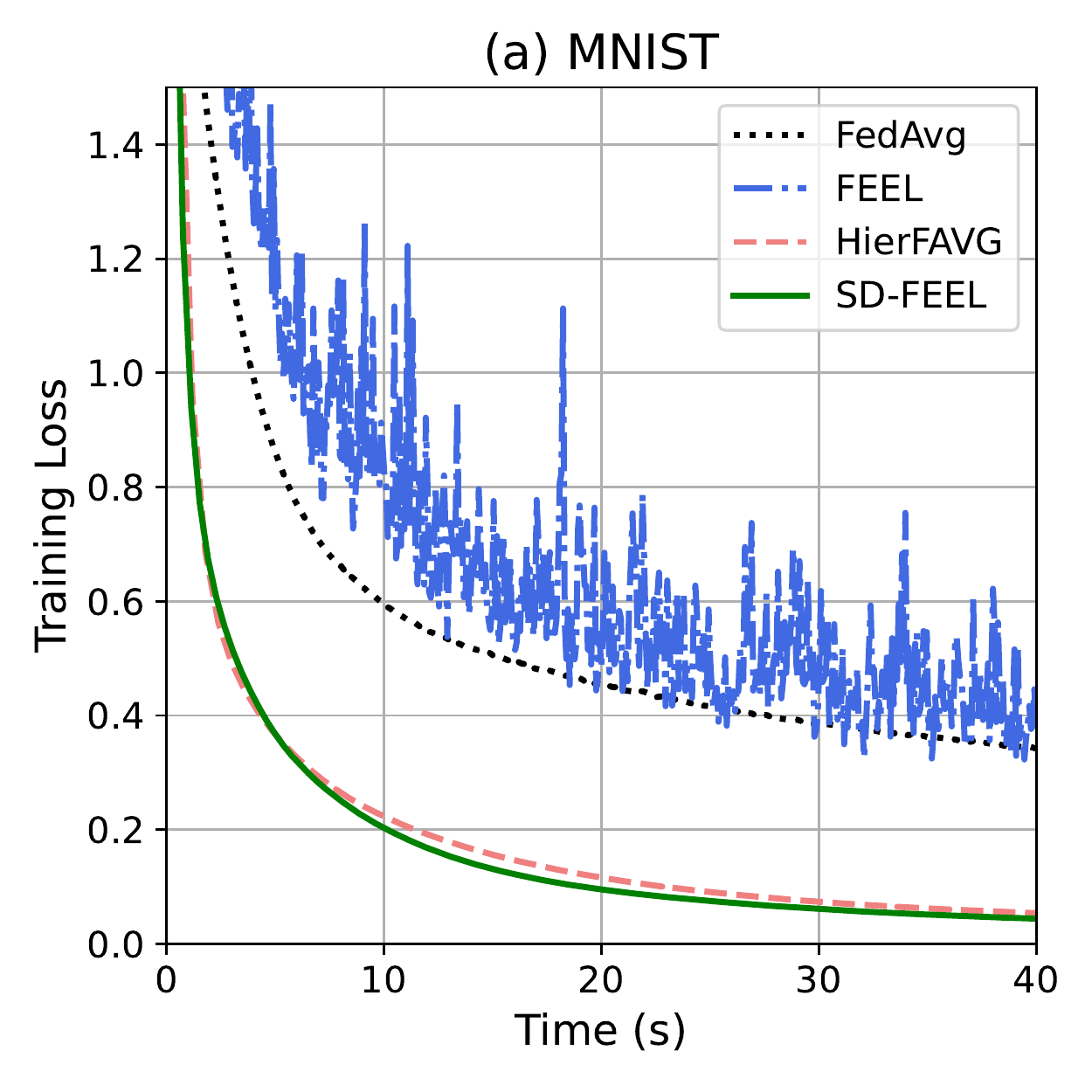}
\end{minipage}%
}%
\subfigure{
\begin{minipage}[t]{0.5\linewidth}
\centering
\includegraphics[width=1\linewidth]{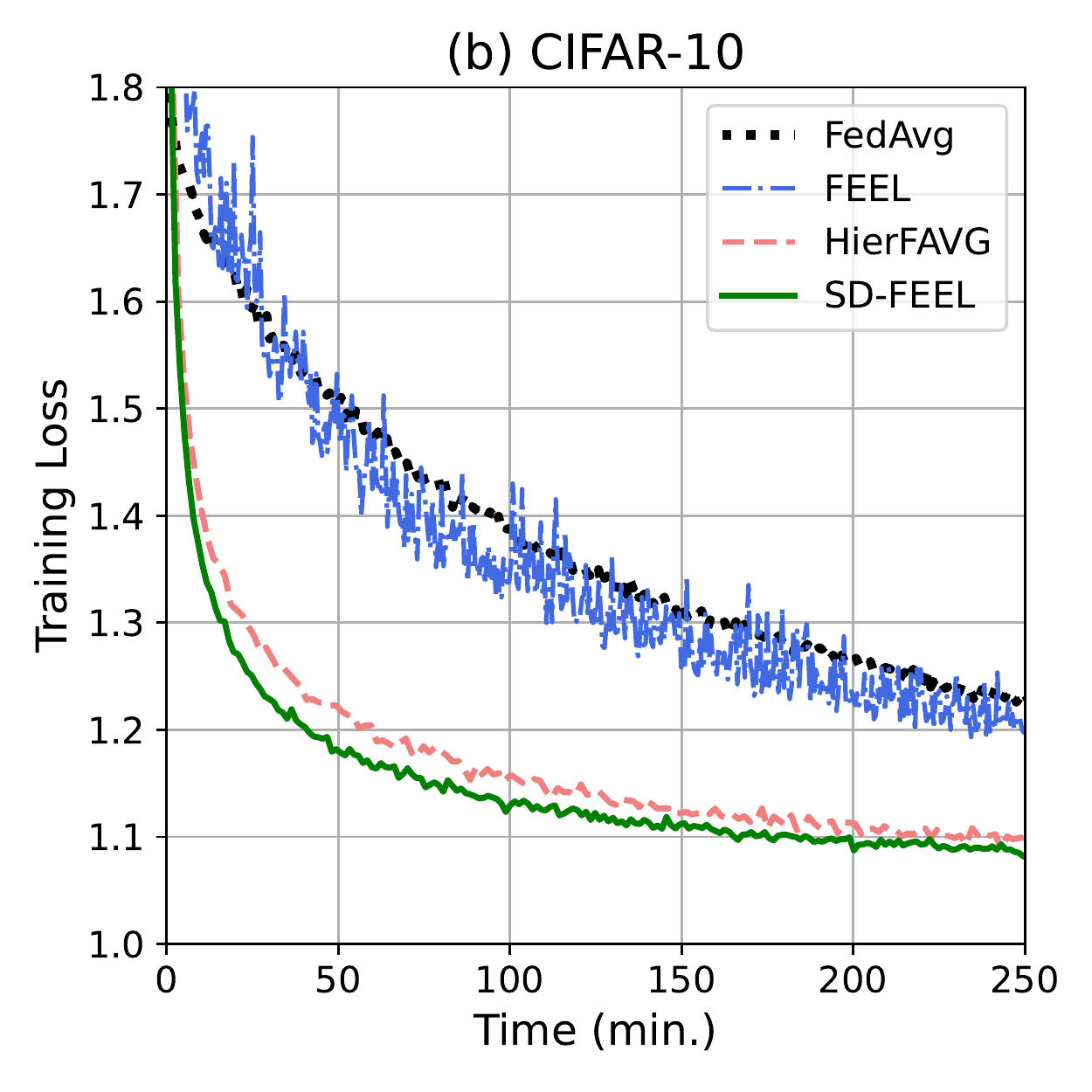}
\end{minipage}
}%
\centering
\caption{Training loss over time on the (a) MNIST ($\tau_1=5$, $\tau_2=1$, and $\alpha=1$) and (b) CIFAR-10 ($\tau_1=2$, $\tau_2=1$, and $\alpha=5$) datasets.}
\label{fig:convergence}
% \vspace{-1.1em}
\end{figure}

\begin{figure}[!t]
\centering
\subfigure{
\begin{minipage}[t]{0.5\linewidth}
\centering
\includegraphics[width=1\linewidth]{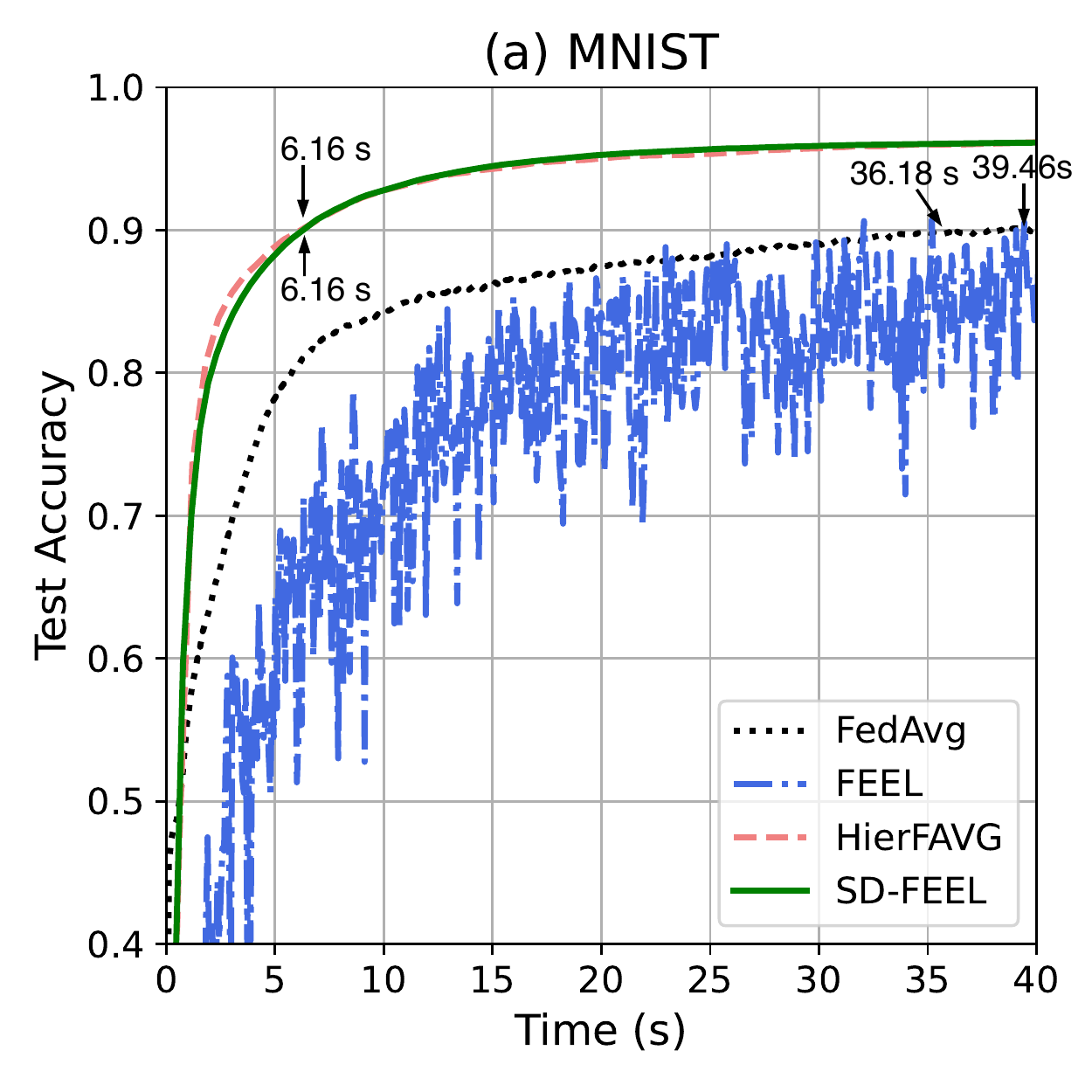}
\end{minipage}%
}%
\subfigure{
\begin{minipage}[t]{0.5\linewidth}
\centering
\includegraphics[width=1\linewidth]{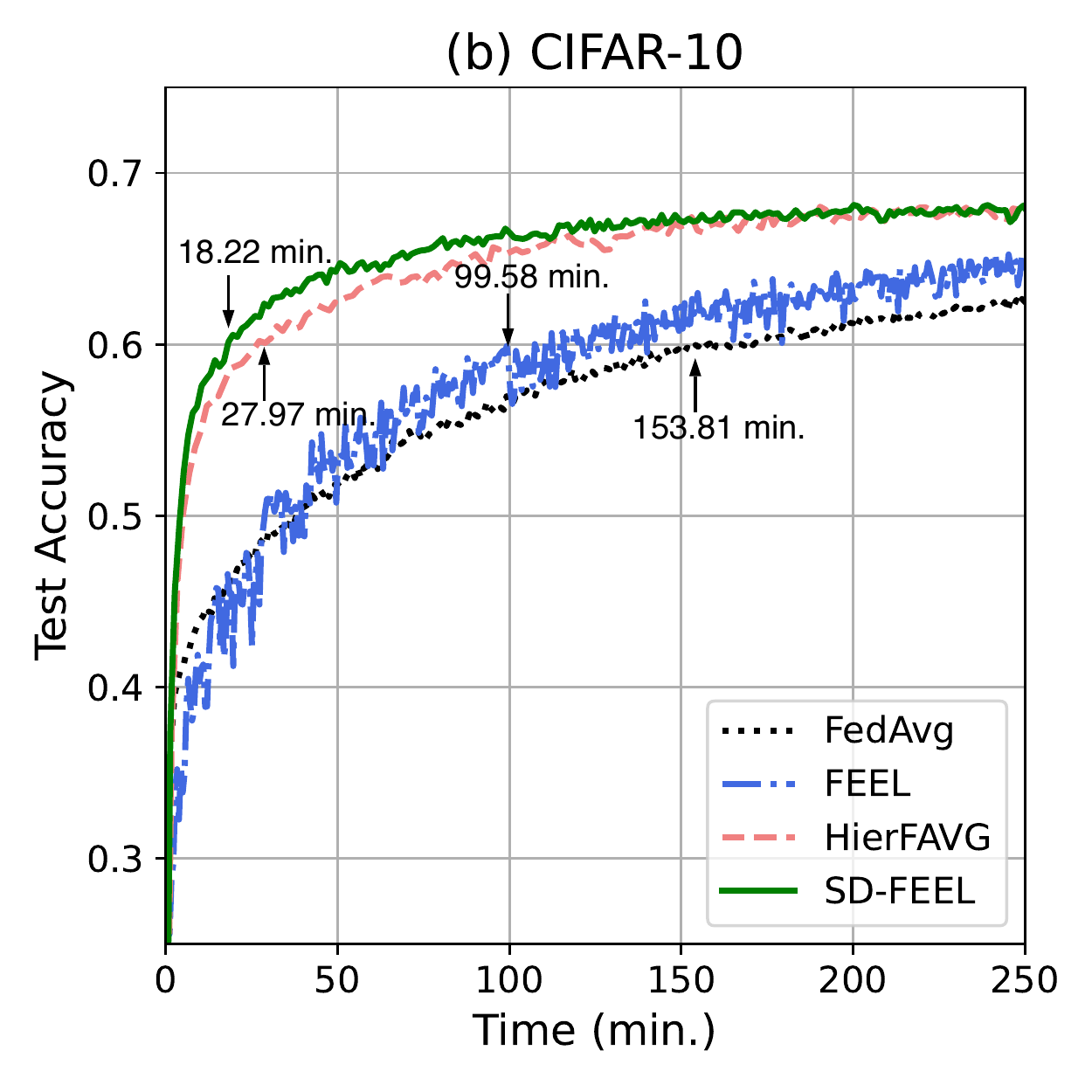}
\end{minipage}
}%
\centering
\caption{Test accuracy over time on the (a) MNIST ($\tau_1=5$, $\tau_2=1$, and $\alpha=1$) and (b) CIFAR-10 ($\tau_1=2$, $\tau_2=1$, and $\alpha=5$) datasets. The training latencies to achieve a target test accuracy of 90\% (respectively 60\%) for the MNIST (respectively CIFAR-10) dataset are also highlighted.}
\label{fig:accuracy}
\end{figure}

To further compare the performance of SD-FEEL and HierFAVG, Fig. \ref{fig:rate}(a) shows the test accuracy over time with different inter-server communication rates, i.e., $R_\text{comm}^\text{sr-sr}=10$ Mbps, $50$ Mbps, and $200$ Mbps. When edge servers share models with a slower rate (e.g., $10$ Mbps), SD-FEEL reaches a lower test accuracy than HierFAVG. Correspondingly, a high communication speed among edge servers (e.g., $200$ Mbps) ensures SD-FEEL converges faster.
Besides, in Fig. \ref{fig:rate}(b), we see that with a sparsely-connected network (i.e., the ring topology), SD-FEEL may have a slower convergence speed due to the model inconsistency among edge servers, which can be alleviated through multiple rounds of communication in inter-cluster model aggregation.
These observations verify the discussion in Remark \ref{rm:reduce}.

\begin{figure}[!t]
\centering
\subfigure{
\begin{minipage}[t]{0.5\linewidth}
\centering
\includegraphics[width=1\linewidth]{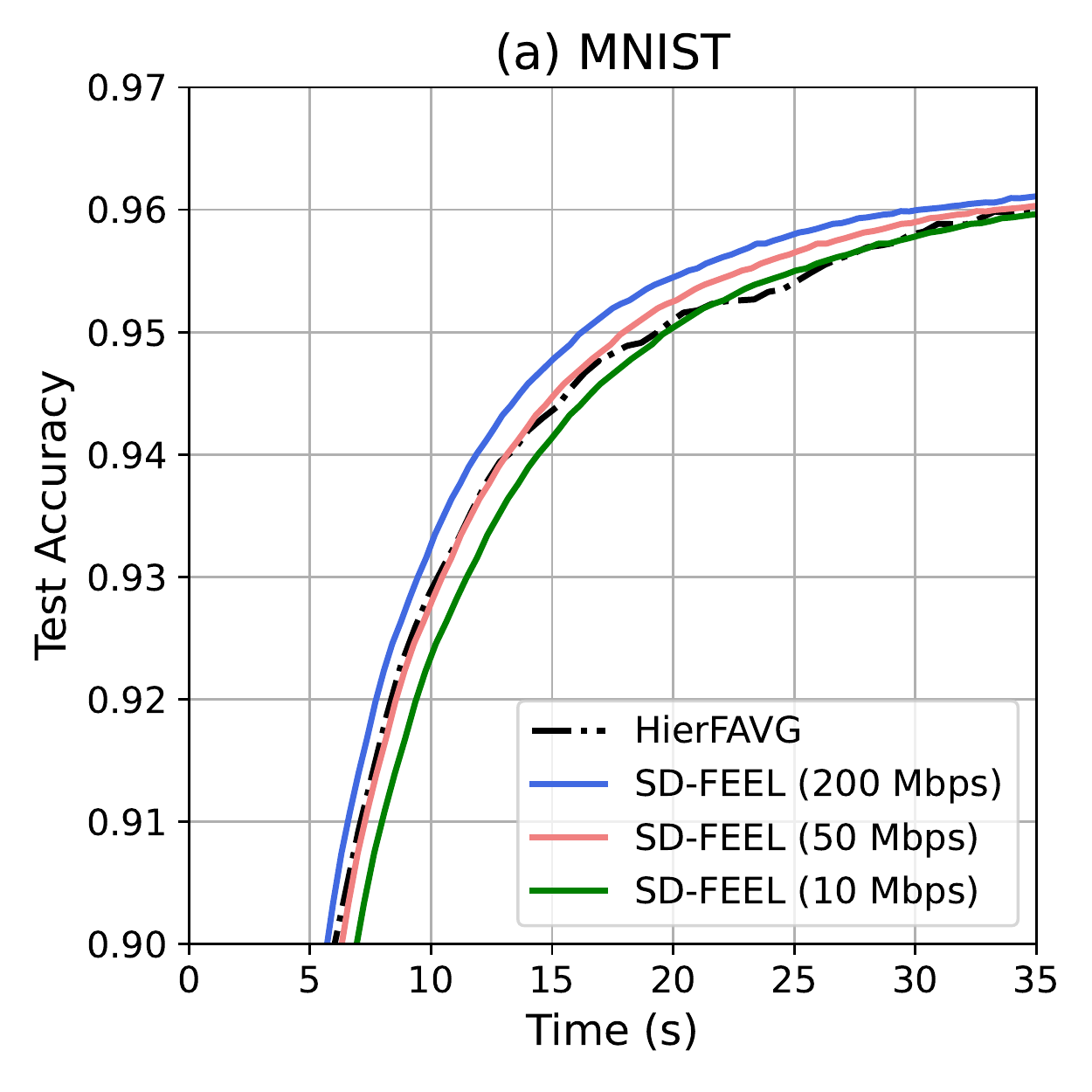}
\end{minipage}%
}%
\subfigure{
\begin{minipage}[t]{0.5\linewidth}
\centering
\includegraphics[width=1\linewidth]{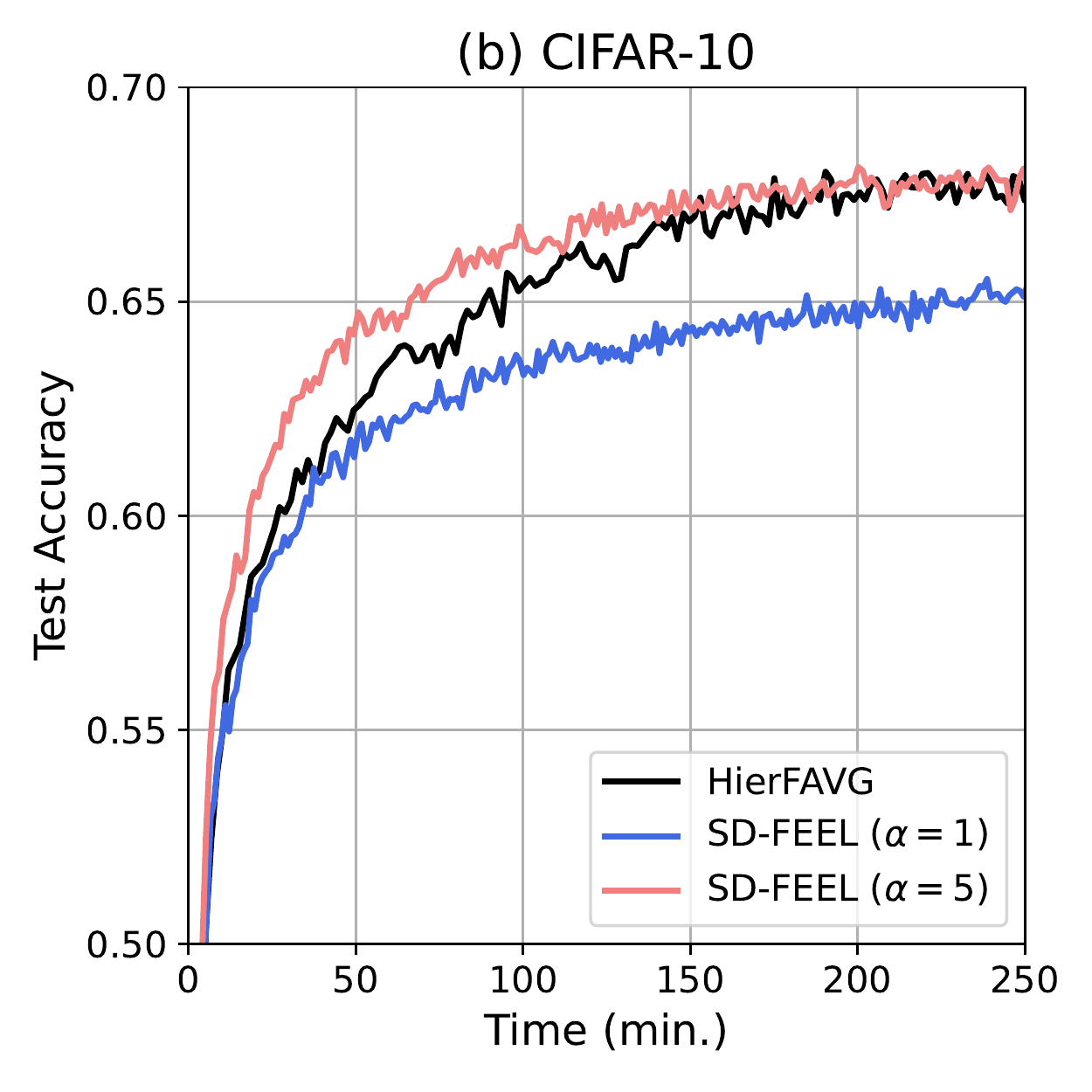}
\end{minipage}
}%
\centering
\caption{Test accuracy over time on the (a) MNIST ($\tau_1=1$, $\tau_2=1$, and $\alpha=1$) and (b) CIFAR-10 ($\tau_1=2$, $\tau_2=1$, $\alpha=1$, and $R_\text{comm}^\text{sr-sr}=50\,\text{Mbps}$) datasets.}
\label{fig:rate}
\end{figure}

\subsubsection{Impacts of Parameters}
We investigate how the aggregation period $\tau_1$ affects the learning performance on the MNIST dataset by showing the relationship between the training loss and training iterations (respectively training time) in Fig. \ref{fig:cifar-tau-1}(a) (respectively Fig. \ref{fig:cifar-tau-1}(b)). We fix $\tau_2=1$ and evaluate SD-FEEL at $\tau_1=1,3,$ and $20$. Fig. \ref{fig:cifar-tau-1}(a) shows that a smaller value of $\tau_1$ leads to a lower training loss within the given training iterations, as explained in Remark \ref{rm:tau}. This conclusion, however, is invalid when considering the training time, as shown in Fig. \ref{fig:cifar-tau-1}(b). Since less frequent communications between client nodes and edge servers can reduce the total latency, a larger value of $\tau_1$ may be preferred. The inter-cluster model aggregation period $\tau_2$ has similar behaviors, and the results are omitted due to space limitation. 

\begin{figure}[!t]
\centering
\subfigure{
\begin{minipage}[t]{0.5\linewidth}
\centering
\includegraphics[width=1\linewidth]{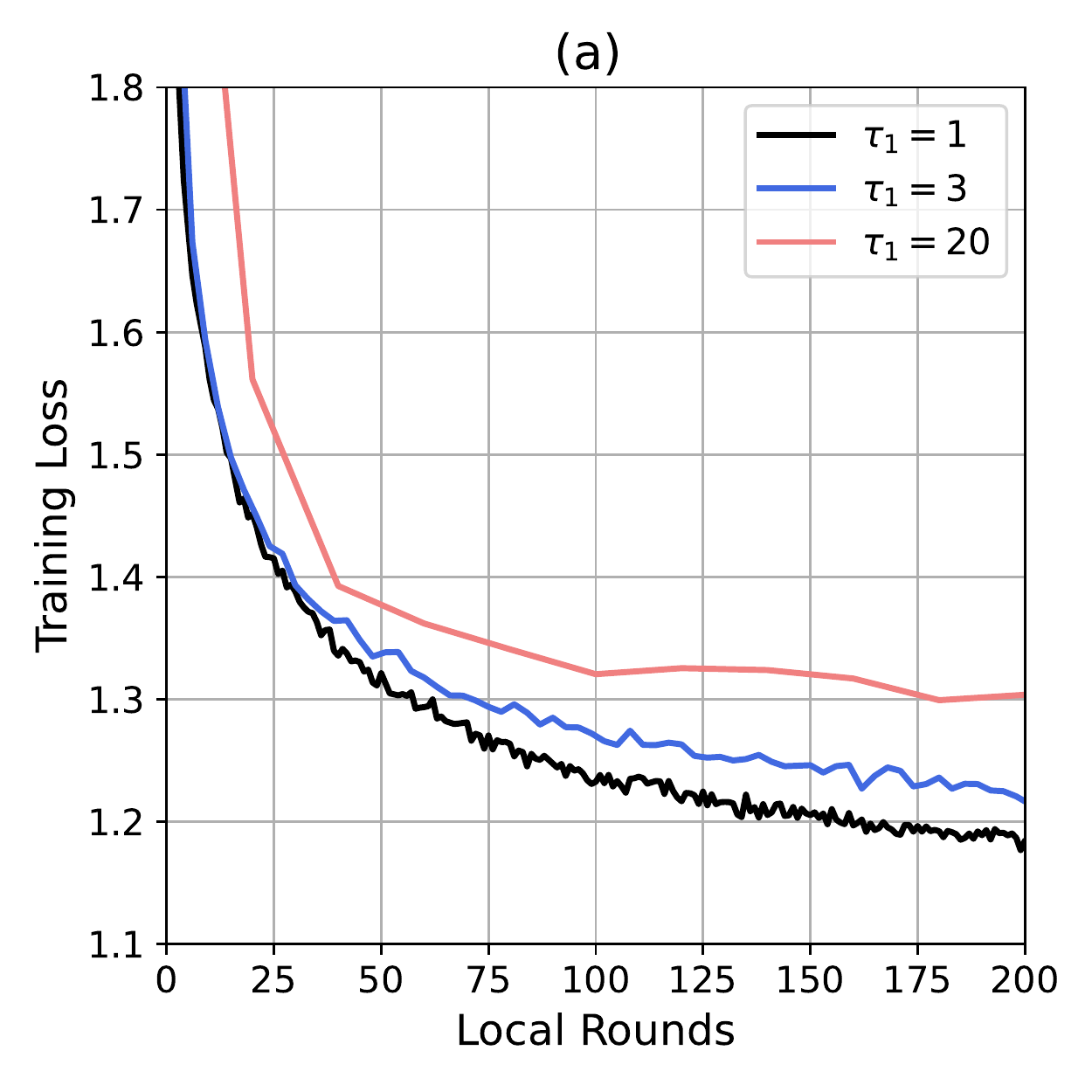}
\end{minipage}%
}%
\subfigure{
\begin{minipage}[t]{0.5\linewidth}
\centering
\includegraphics[width=1\linewidth]{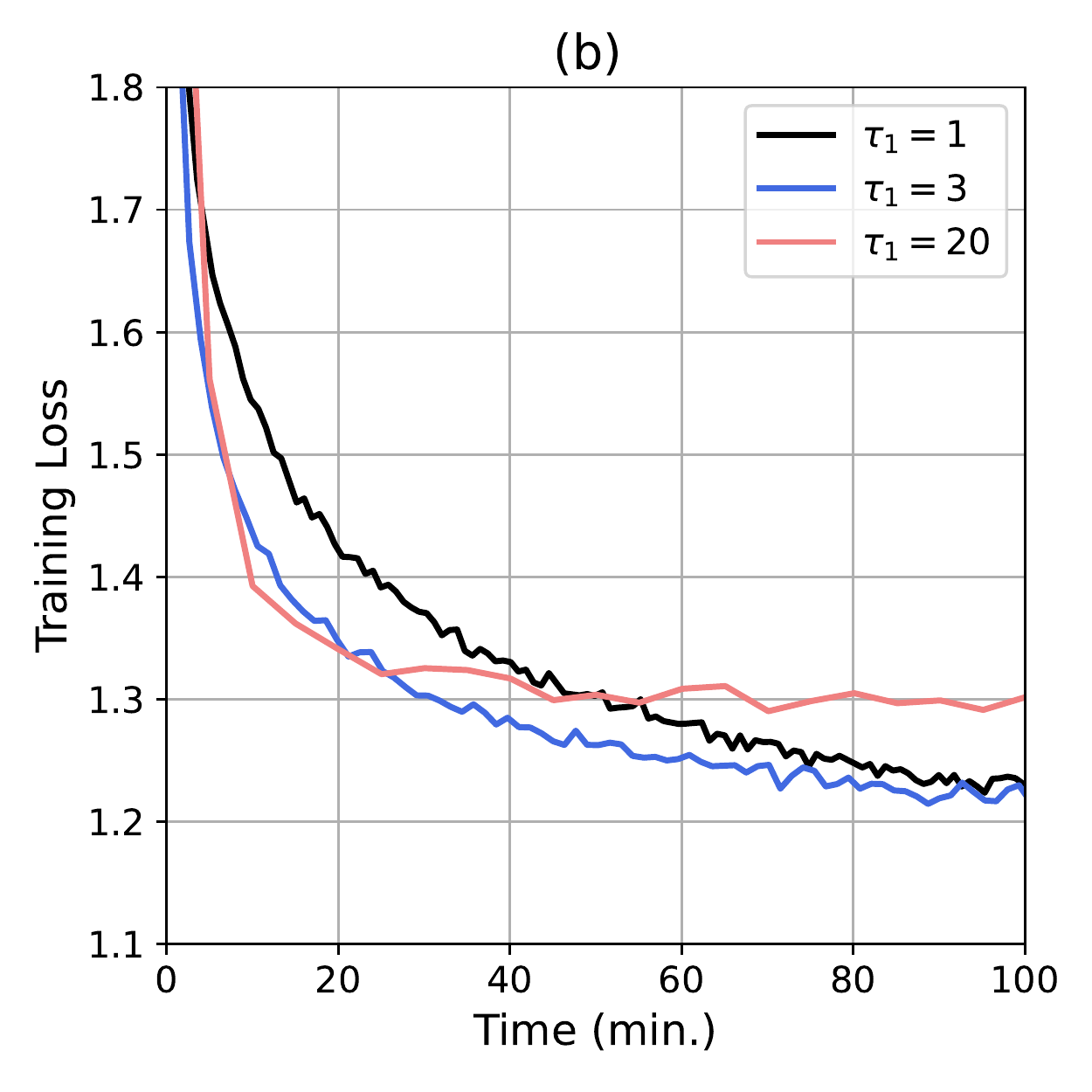}
\end{minipage}
}%
\centering
\caption{Training loss of SD-FEEL ($\tau_2=1$ and $\alpha=1$) over (a) local iterations and (b) time on the CIFAR-10 dataset.}
\label{fig:cifar-tau-1}
\end{figure}

We also evaluate the test accuracy of SD-FEEL on different network topologies of the edge servers. As shown in Fig. \ref{fig:topo}, with $\alpha=1$ round of communication, a more connected topology leads to a higher test accuracy within the given number of training iterations. This is because more information can be received from neighboring edge clusters in each round of inter-cluster model aggregation, as discussed in Remark \ref{rm:topo}. Besides, as $\alpha$ increases in the ring topology, the training speed becomes faster in terms of training iterations since more information is collected from neighboring edge clusters. When $\alpha=10$ (respectively $\alpha=15$), SD-FEEL with the ring topology leads to a comparable performance with a fully-connected topology on the MNIST dataset (respectively CIFAR-10 dataset), which corroborates the discussion in Remark \ref{rm:topo}.

\begin{figure}[!t]
\centering
\subfigure{
\begin{minipage}[t]{0.5\linewidth}
\centering
\includegraphics[width=1\linewidth]{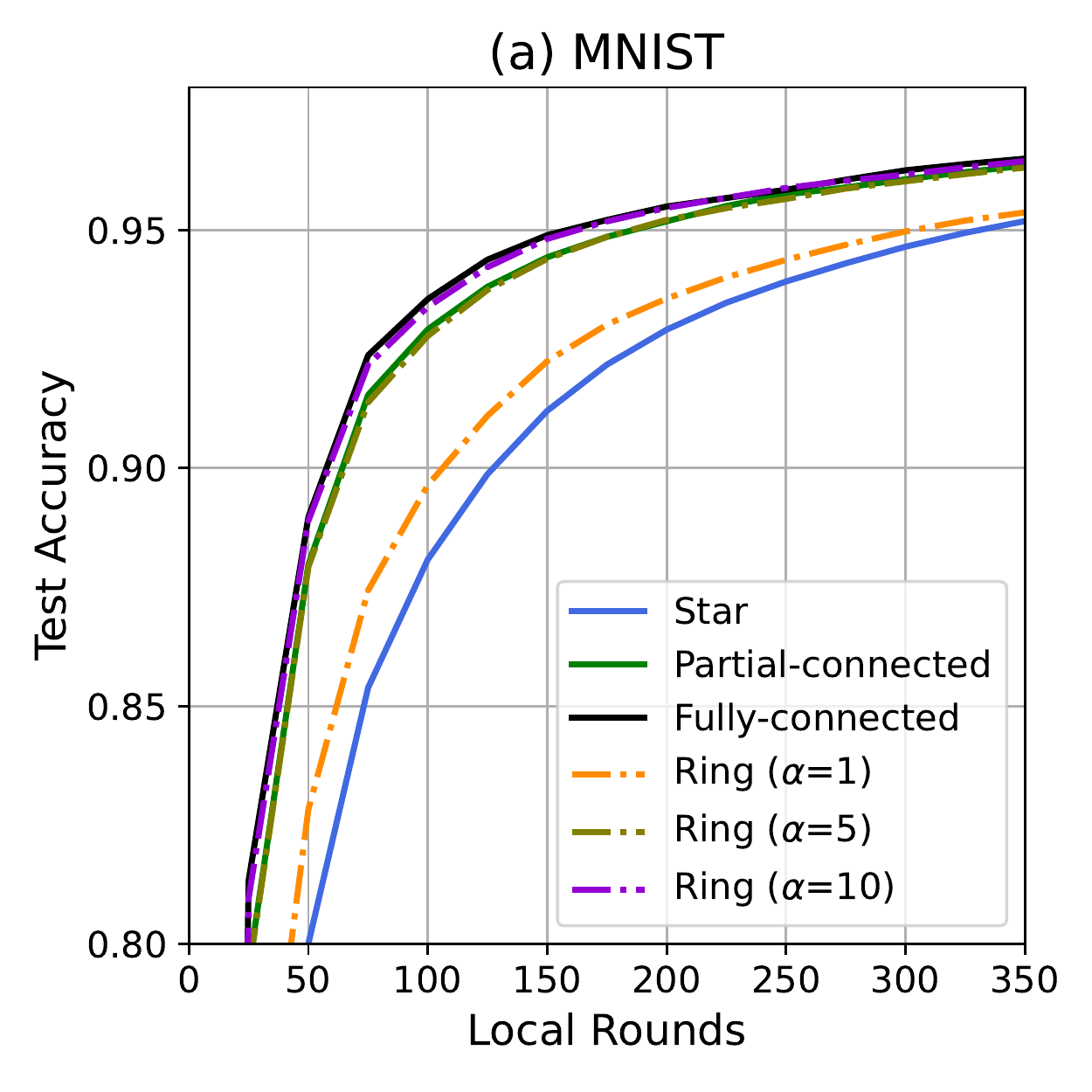}
\end{minipage}%
}%
\subfigure{
\begin{minipage}[t]{0.5\linewidth}
\centering
\includegraphics[width=1\linewidth]{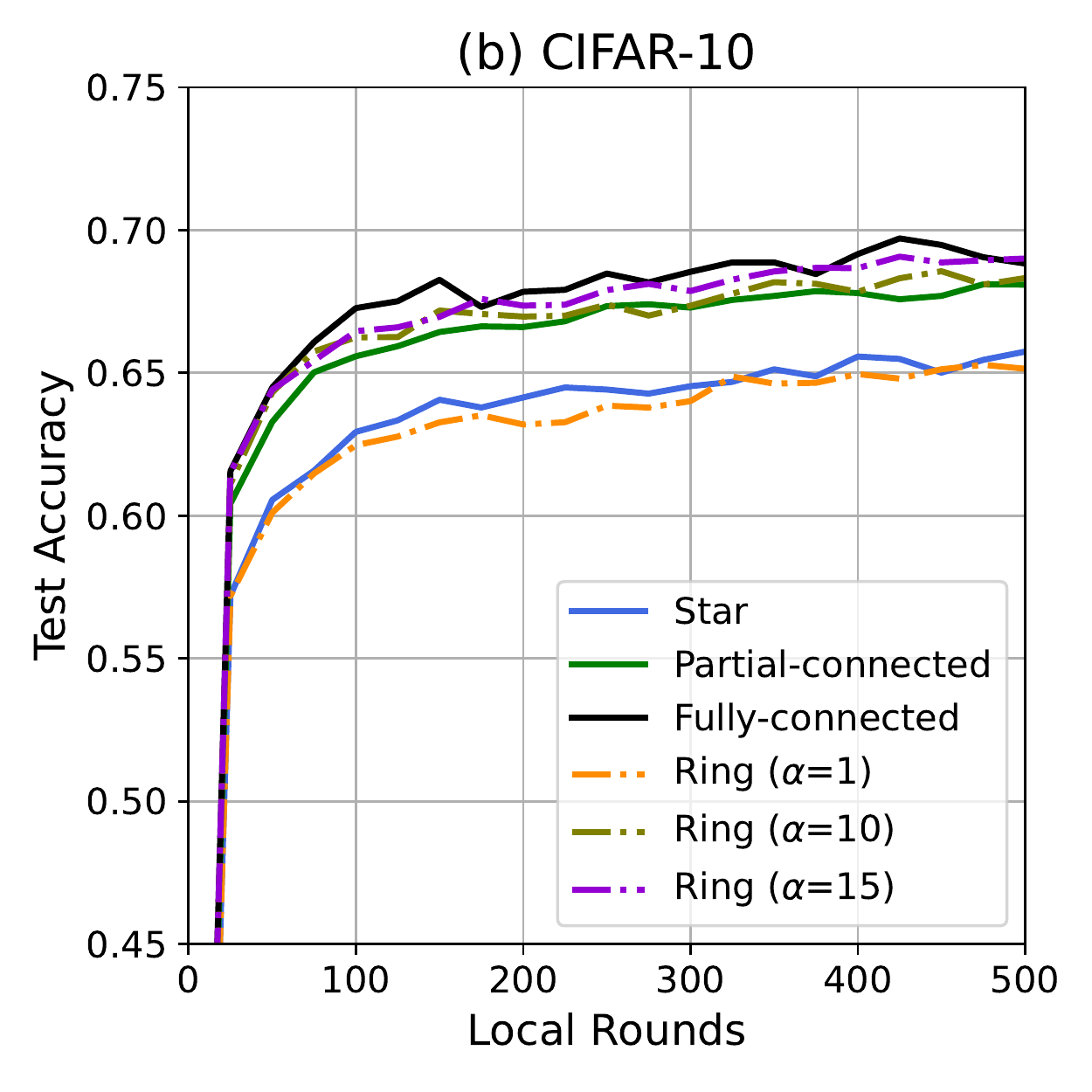}
\end{minipage}
}%
\centering
\caption{Test accuracy over local iterations with different network topologies ($\tau_1=5$ and $\tau_2=5$) on the (a) MNIST and (b) CIFAR-10 datasets.}
\label{fig:topo}
\end{figure}

\begin{figure}[!t]
\centering
\subfigure{
\begin{minipage}[t]{0.5\linewidth}
\centering
\includegraphics[width=1\linewidth]{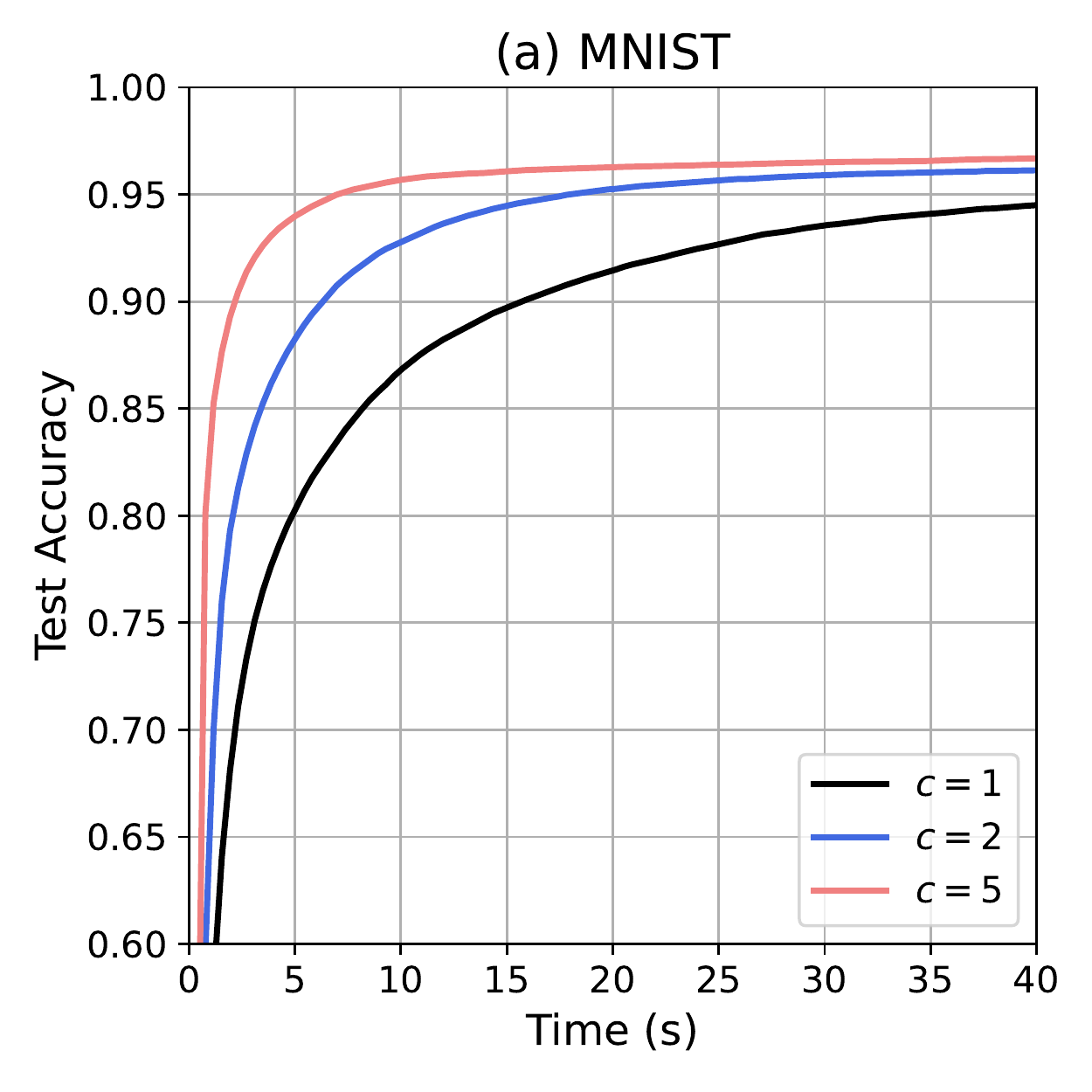}
\end{minipage}%
}%
\subfigure{
\begin{minipage}[t]{0.5\linewidth}
\centering
\includegraphics[width=1\linewidth]{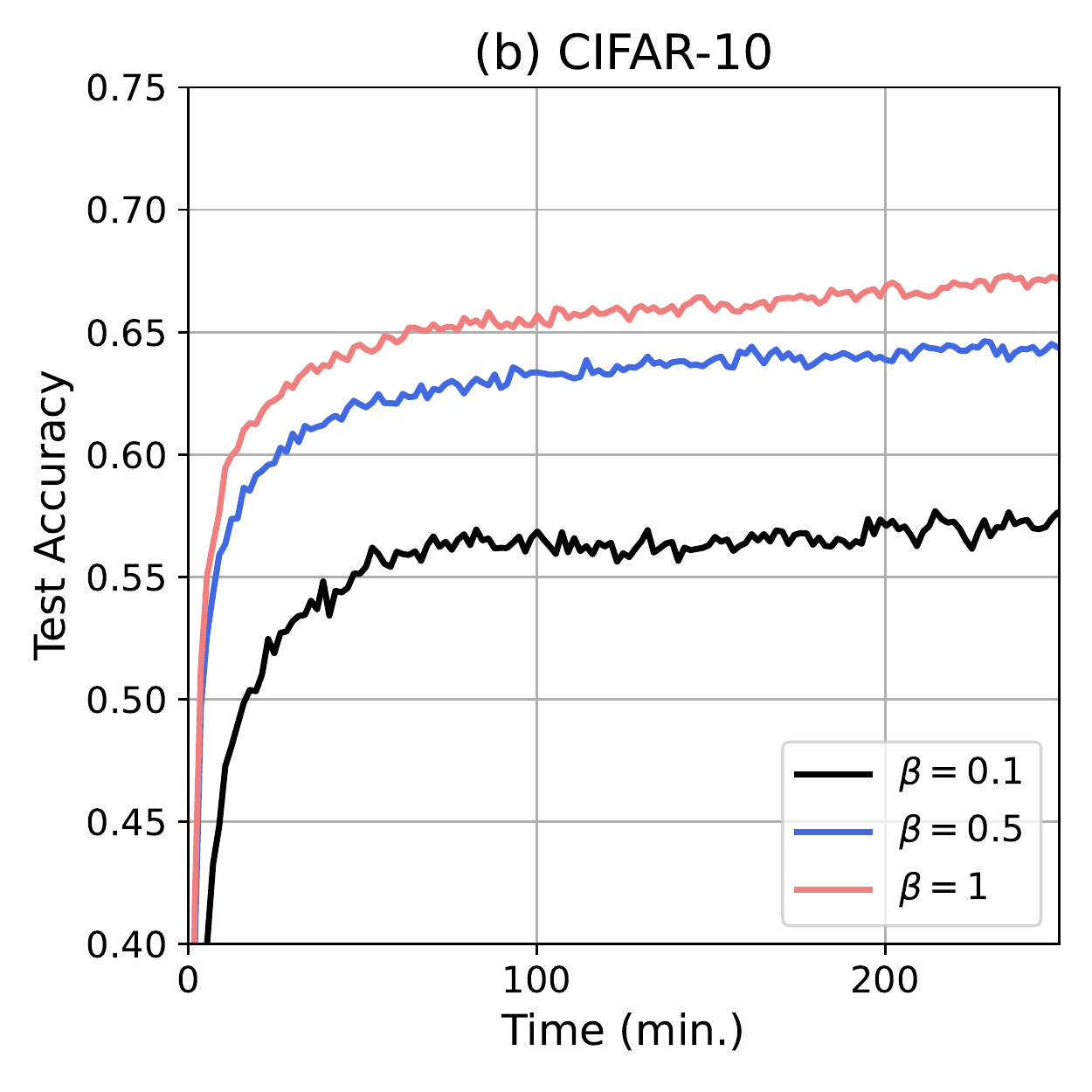}
\end{minipage}
}%
\centering
\caption{Test accuracy over time with varying degrees of non-IIDness ($\tau_1=5$, $\tau_2=1$, $\alpha=1$, and $H=1$) on the (a) MNIST and (b) CIFAR-10 datasets.}
\label{fig:non-iid}
\end{figure}

\subsubsection{Effect of Data and Device Heterogeneity}
We test SD-FEEL with different degrees of data and device heterogeneity. We start with the effect of data heterogeneity in Fig. \ref{fig:non-iid}.
According to Fig. \ref{fig:non-iid}(a), when each client node has more classes of data samples, the local training is less biased and thus SD-FEEL has a faster learning speed. Similarly in Fig. \ref{fig:non-iid}(b), a smaller value of $\beta$ leads to a higher degree of data heterogeneity and thus slows down the training process significantly.

To investigate the effect of device heterogeneity, we next compare synchronous SD-FEEL (denoted by \textit{Sync.}), asynchronous SD-FEEL (denoted by \textit{Async.}), and asynchronous SD-FEEL with a constant mixing matrix (denoted by \textit{Vanilla Async.}) in Fig. \ref{fig:async}(a). 
The computation deadline is chosen such that each client node can compute at least 100 (respectively 1000) batches of data samples on the MNIST (respectively CIFAR-10) dataset. 
Besides, we adopt $\psi(\delta_t^{(j)})=\frac{1}{2 (\delta_t^{(j)}+1)}$ to calculate the mixing matrix for inter-cluster model aggregation.
According to Fig. \ref{fig:async}(a), we observe that asynchronous SD-FEEL with a constant mixing matrix has a slightly slower convergence than synchronous training. Nevertheless, with the proposed staleness-aware mixing matrix, asynchronous SD-FEEL performs much better, which demonstrates the necessity of our proposed staleness-aware mixing matrix.

% \begin{figure}[!t]
% % \setlength{\belowcaptionskip}{-100pt} 
% \setlength\abovecaptionskip{0pt}
% \centering
% \subfigure[MNIST]{
% \begin{minipage}[t]{0.45\linewidth}
% \centering
% \includegraphics[width=1\linewidth]{mnist-async-acc.pdf}
% %\caption{fig1}
% \end{minipage}%
% }%
% \subfigure[CIFAR-10]{
% \begin{minipage}[t]{0.45\linewidth}
% \centering
% \includegraphics[width=1\linewidth]{cifar-async-acc.pdf}
% %\caption{fig2}
% \end{minipage}%
% }%
% \centering
% \caption{Test accuracy over time on the (a) MNIST and (b) CIFAR-10 datasets.}
% \label{fig:async}
% % \vspace{-1.3em}
% \end{figure}

\begin{figure}[!t]
\centering
\subfigure{
\begin{minipage}[t]{0.5\linewidth}
\centering
\includegraphics[width=1\linewidth]{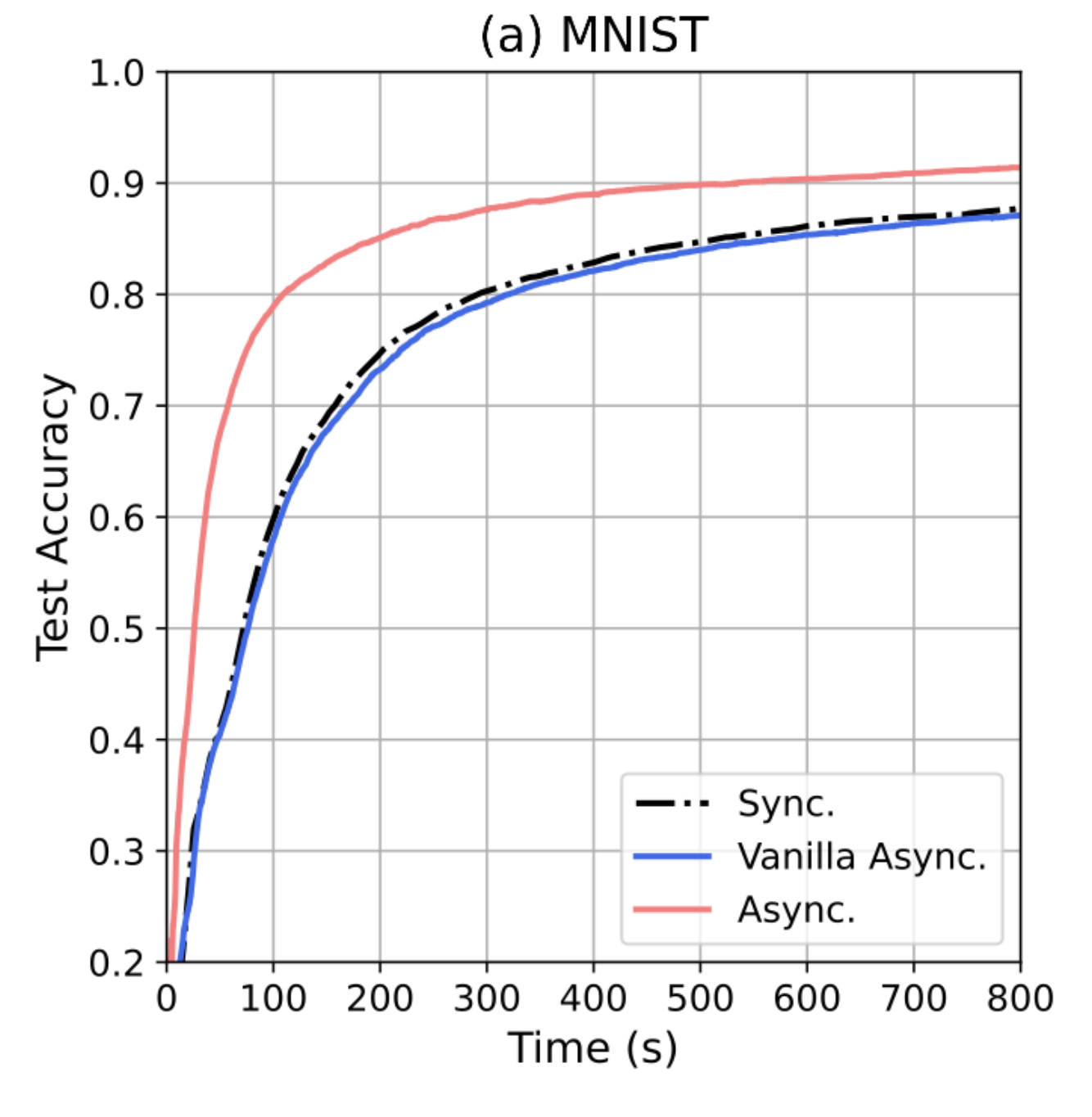}
\end{minipage}%
}%
\subfigure{
\begin{minipage}[t]{0.5\linewidth}
\centering
\includegraphics[width=1\linewidth]{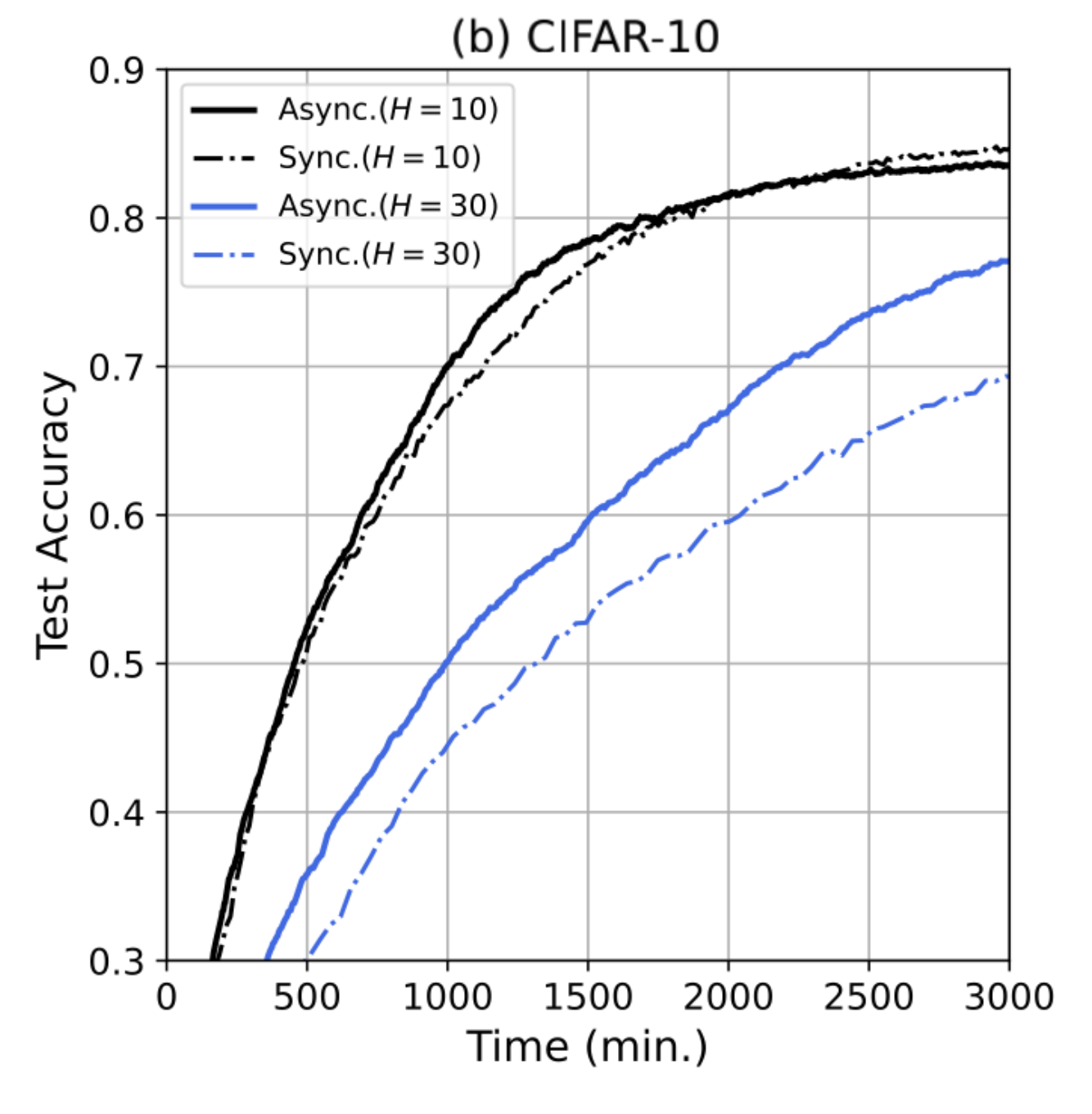}
\end{minipage}
}%
\centering
\caption{Test accuracy over time on the (a) MNIST and (b) CIFAR-10 datasets.}
\label{fig:async}
\end{figure}

Fig. \ref{fig:async} (b) shows the test accuracy over time with different degrees of device heterogeneity on the CIFAR-10 dataset. When the computational resources are quite uneven (i.e., with a larger $H$), the convergence speed of synchronous SD-FEEL is degraded, since slower client nodes require more time for local training. Comparatively, the training efficiency of SD-FEEL is improved with the asynchronous training algorithm, which is more notable as the device heterogeneity becomes more significant. This is because faster client nodes are allowed to perform more local training and thus have less idle time. 
With the limited training time, asynchronous SD-FEEL obtains an improvement in test accuracy over synchronous training. 
However, if given a sufficiently long training time such that the data at the slower client nodes are fully exploited, synchronous SD-FEEL is able to achieve a higher test accuracy.

\subsubsection{Effect of Learning Rate}

% \begin{figure}[!t]
% % \setlength{\belowcaptionskip}{-100pt} 
% \setlength\abovecaptionskip{0pt}
% \centering
% \subfigure[]{
% \begin{minipage}[t]{0.45\linewidth}
% \centering
% \includegraphics[width=1\linewidth]{acc-vary-lr.pdf}
% %\caption{fig1}
% \end{minipage}%
% }%
% \subfigure[]{
% \begin{minipage}[t]{0.45\linewidth}
% \centering
% \includegraphics[width=1\linewidth]{acc-unequal.pdf}
% %\caption{fig2}
% \end{minipage}%
% }%
% \centering
% \caption{Test accuracy over local iterations with (a) varying learning rates and (b) different cluster imbalance on the MNIST dataset ($\tau_1=5,\tau_2=1$, and $\alpha=1$).}
% \label{fig:lr}
% \end{figure}

\begin{figure}[!t]
\centering
\subfigure{
\begin{minipage}[t]{0.5\linewidth}
\centering
\includegraphics[width=1\linewidth]{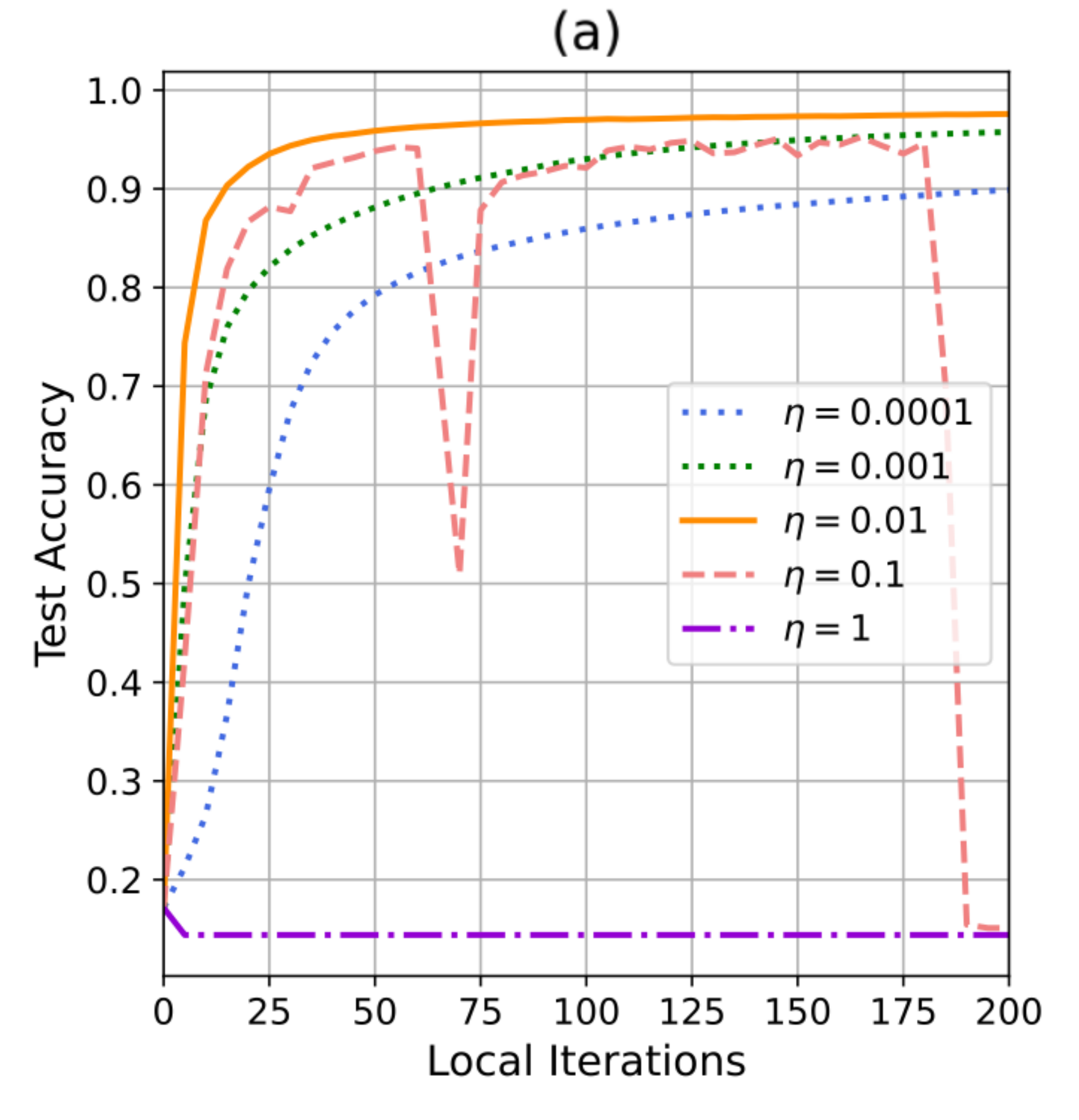}
\end{minipage}%
}%
\subfigure{
\begin{minipage}[t]{0.5\linewidth}
\centering
\includegraphics[width=1\linewidth]{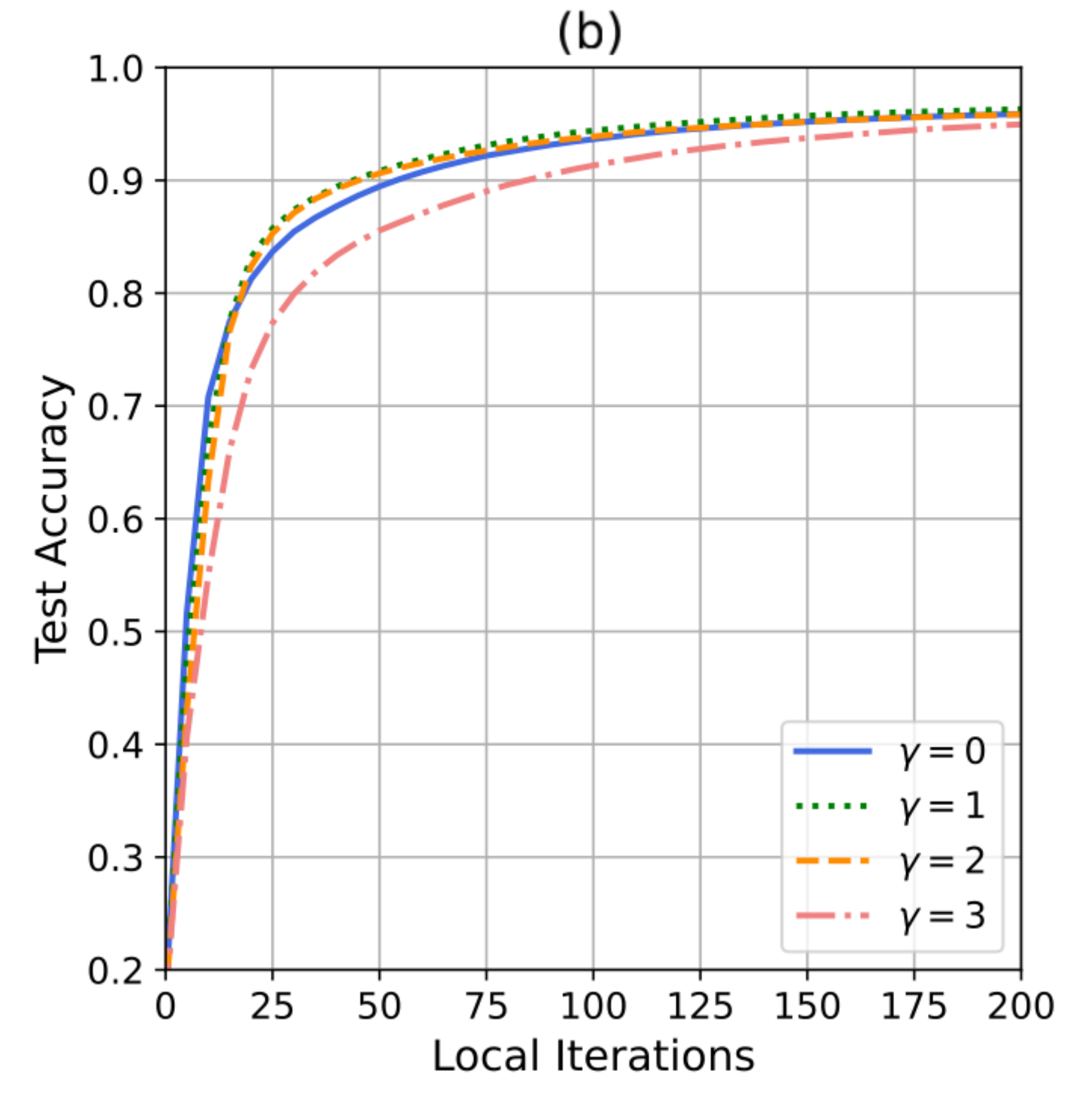}
\end{minipage}
}%
\centering
\caption{Test accuracy over local iterations with (a) varying learning rates and (b) different cluster imbalance on the MNIST dataset ($\tau_1=5,\tau_2=1$, and $\alpha=1$).}
\label{fig:lr}
\end{figure}

We evaluate the effect of the learning rate on the MNIST training task in Fig. \ref{fig:lr}(a).
The results show that as the learning rate increases from $0.0001$ to $0.01$, SD-FEEL achieves faster convergence.
This is because a large learning rate allows a faster step towards the global optimal.
Nevertheless, when the learning rate is too large, i.e., $\eta=0.1$ or $\eta=1$, the training process becomes unstable and may even suffer from divergence.
Therefore, the learning rate should be selected to achieve a proper convergence speed and avoid training failure.

\subsubsection{Effect of Cluster Imbalance}

We define a parameter $\gamma$ to represent the level of cluster imbalance.
Specifically, we assume four of the total ten edge clusters have five client nodes, three edge clusters have $5-\gamma$ client nodes, and others have $5+\gamma$ client nodes.
It is straightforward that a larger value of $\gamma$ implies the training data among edge clusters is more unbalanced.
Fig. \ref{fig:lr}(b) shows the learning performance with different values of $\gamma$, from which, we see the convergence speed of SD-FEEL remains similar despite with a slight cluster imbalance.
Nevertheless, a severe cluster imbalance, i.e., $\gamma=3$, results in much slower convergence.
However, different cluster imbalance setups converge to the same model accuracy, since the information of training data can be exchanged among edge clusters through sufficient inter-cluster model aggregations.
This also verifies the necessity of cooperation among edge servers.

%%%%%%%%%%%%%%%%%%%%%%%%%%%%%%%%%%%%%%%%%%
%%%%%%%%%%%%%%% Conclusion %%%%%%%%%%%%%%%
%%%%%%%%%%%%%%%%%%%%%%%%%%%%%%%%%%%%%%%%%%
\section{Conclusions}\label{sec:conclusion}

In this paper, we investigated a novel FL system for privacy-preserving distributed learning in 6G networks, named semi-decentralized federated edge learning (SD-FEEL). It enjoys high training efficiency by employing low-latency communication among multiple edge servers.
We presented the training algorithm for SD-FEEL with a convergence analysis on non-IID data. Then the effects of various parameters on the training performance of SD-FEEL were discussed.
Moreover, to combat the device heterogeneity, we proposed an asynchronous training algorithm for SD-FEEL, followed by its convergence analysis.
Simulation results demonstrated the benefits of SD-FEEL and corroborated our analysis.
For future works, it is worth considering the performance of SD-FEEL under varying channel conditions.
Further investigations are also needed for the selection of key algorithmic parameters as well as practical deployment.
Extending SD-FEEL and its theoretical analysis to scenarios with a second-order local optimizer is also very interesting.

\appendix\label{proof-appendix}
% \subsection{Proofs} 
To obtain Lemma \ref{lemma-e-t}, we first prove the upper bounds for $\mathcal{E}_{t,1}$, $\mathcal{E}_{t,2}$ and $\mathcal{E}_{t,3}$ in Lemmas \ref{lemma-4}-\ref{lemma-6}.
Then the proof is completed by plugging the results into \eqref{eq:respective}, summing up both sides over $t=0,1,\dots,T-1$, and then dividing them by $T$.

\begin{lemma}\label{lemma-4}
With Assumption \ref{assumptions}, we have:
\begin{equation}
\begin{split}
    \mathcal{E}_{t,1}
    & \leq 2 \eta^2 \delta_\text{max}^2 \theta_\mathrm{max}^2\theta_\text{min}^{-1} \sigma^2 \\
    & + 4 \eta^2 \delta_\text{max}^2 \theta_\mathrm{max}^2\theta_\text{min}^{-1} \kappa^2
    + 4 \eta^2  \delta_\text{max}^2 \theta_\mathrm{max}^2 Q_{\Tilde{t}}.
\end{split}
\end{equation}
\end{lemma}
\begin{proof}
Denote $\nabla \hat{F}_s^{(i)} \triangleq \sum_{i\in\mathcal{C}_d} \frac{\hat{m}_i}{\theta_i} \sum_{l=0}^{\theta_i-1} \nabla F_i(\bm{w}_{s,l}^{(i)})$.
We first provide an upper bound for a term $\mathbb{E} \left\| \nabla \mathbf{\hat{F}}_{s} \right\|_{\mathbf{\Tilde{M}}}^2$ where $\mathbf{\Tilde{M}} \triangleq \bm{\Tilde{m}}\mathbf{1}_D^\TT$.
\begin{align}
    &\quad \mathbb{E} \left[ \left\| \nabla \mathbf{\hat{F}}_{s} \right\|_{\mathbf{\Tilde{M}}}^2 \right] \nonumber \\
    &= \sum_{d\in\mathcal{D}} \Tilde{m}_d \mathbb{E} \left[ \bigg\| \nabla \hat{F}_s^{(i)} - \sum_{i\in\mathcal{C}} m_i \nabla \hat{F}_s^{(i)} + \sum_{i\in\mathcal{C}} m_i \nabla \hat{F}_s^{(i)} \bigg\|^2 \right] \nonumber \\
    &\overset{(a)}{\leq} \sum_{d\in\mathcal{D}} \Tilde{m}_d \bigg( 2\mathbb{E} \left[ \bigg\| \nabla \hat{F}_s^{(i)} - \sum_{i\in\mathcal{C}} m_i \nabla \hat{F}_s^{(i)} \bigg\|^2 \right] \nonumber \\
    & + 2\mathbb{E} \left[ \bigg\| \sum_{i\in\mathcal{C}} m_i \nabla \hat{F}_s^{(i)} \bigg\|^2  \right] \bigg) \nonumber \\
    &\overset{(b)}{\leq} \sum_{d\in\mathcal{D}} \Tilde{m}_d \bigg( 2\sum_{i\in\mathcal{C}} m_i \frac{1}{\theta_i^2} \sum_{l=0}^{\theta_i-1} \kappa^2
    + 2\mathbb{E}  \left[ \bigg\| \sum_{i\in\mathcal{C}} m_i \nabla \hat{F}_s^{(i)} \bigg\|^2 \right] \bigg) \nonumber \\
    &= \sum_{d\in\mathcal{D}} \Tilde{m}_d \bigg( 2\sum_{i\in\mathcal{C}} \frac{m_i}{\theta_i} \kappa^2 + 2Q_s \bigg) \nonumber \\
    &\leq 2\theta_\text{min}^{-1}\kappa^2 + 2Q_s, \label{eq:delta-F}
\end{align}
where (a) follows the Jensen's inequality, and (b) is due to \eqref{eq-kappa} in Assumption \ref{assumptions}.
Then we conclude the proof as follows:
\begin{align}
    & \quad \mathcal{E}_{t,1} \nonumber\\
    &= \mathbb{E} \left[ \bigg\| \eta \sum_{s=1}^{\delta_t^{(d)}}  \overline{\theta}_{t-s} \mathbf{\hat{G}}_{t-s} \bm{\Tilde{m}} \bigg\|^2 \right] \nonumber\\
    &\leq \eta^2 \delta_t^{(d)} \theta_\mathrm{max}^2  \sum_{s=1}^{\delta_t^{(d)}} \mathbb{E}  \left[ \left\| \mathbf{\hat{G}}_{t-s} - \nabla \mathbf{\hat{F}}_{t-s} + \nabla \mathbf{\hat{F}}_{t-s} \right\|_\mathbf{\Tilde{M}}^2 \right] \nonumber\\
    &\leq 2\eta^2 \delta_t^{(d)} \theta_\mathrm{max}^2  \sum_{s=1}^{\delta_t^{(d)}} \mathbb{E} \left[ \left\| \mathbf{\hat{G}}_{t-s} - \nabla \mathbf{\hat{F}}_{t-s} \right\|_\mathbf{\Tilde{M}}^2 \right] \nonumber\\
    &\quad + 2\eta^2 \delta_t^{(d)} \theta_\mathrm{max}^2  \sum_{s=1}^{\delta_t^{(d)}} \mathbb{E} \left[  \left\| \nabla \mathbf{\hat{F}}_{t-s} \right\|_\mathbf{\Tilde{M}}^2 \right] \nonumber\\
    &\overset{(c)}{\leq} 2 \eta^2 \delta_\text{max}^2 \theta_\mathrm{max}^2 \theta_\mathrm{min}^{-1} \sigma^2 
    + 2\eta^2 \delta_\text{max} \theta_\mathrm{max}^2  \sum_{s=1}^{\delta_t^{(d)}} \mathbb{E} \left[ \left\| \nabla \mathbf{\hat{F}}_{t-s} \right\|_\mathbf{\Tilde{M}}^2 \right] \nonumber\\
    &\overset{(d)}{\leq} 2 \eta^2 \delta_\text{max}^2 \theta_\mathrm{max}^2 \theta_\mathrm{min}^{-1} \sigma^2
    + 4 \eta^2 \delta_\text{max}^2 \theta_\mathrm{max}^2 \theta_\mathrm{min}^{-1} \kappa^2
    + 4 \eta^2 \delta_\text{max}^2 \theta_\mathrm{max}^2 Q_{\Tilde{t}},
\end{align}
where (c) follows Assumption \ref{assumptions} and Lemma \ref{lemma-delta}, and (d) follows \eqref{eq:delta-F}.
\end{proof}

%%%%%%%%%%%%%%%%%%%%%%%%%%%%Lemma 5
\begin{lemma}\label{lemma-8}
With Assumption \ref{assumptions}, we have:
\begin{equation}
\begin{split}
    \mathcal{E}_{t,2}
    & \leq 2\eta^2 \frac{\theta_\mathrm{max}^2}{\theta_\text{min}}\sigma^2 Z_{t,\rho}^{(1)}
    + 4\eta^2 \frac{\theta_\mathrm{max}^2}{\theta_\text{min}} \kappa^2 Z_{t,\rho}^{(2)} \\
    & + 4\eta^2 \theta_\mathrm{max}^2 \left( \sum_{l=1}^{t-1} \rho_{l,t-1} \right) \left( \sum_{s=1}^{t-1} \rho_{s,t-1} Q_s \right),
\end{split}
\end{equation}
where $Z_{t,\rho}^{(1)} \triangleq \sum_{s=1}^{t-1} \rho_{s,t-1}^2$, and $Z_{t,\rho}^{(2)} \triangleq \left( \sum_{s=1}^{t-1} \rho_{s,t-1} \right)^2$.
\end{lemma}
\begin{proof}
We first define the evolution of consensus status as follows:
\begin{equation}
    \rho_{s,t-1} \triangleq \left\| \prod_{l=s}^{t-1} \mathbf{P}_{l} - \mathbf{M} \right\|_\text{op} = \prod_{l=s}^{t-1} \left\| \mathbf{P}_{l} - \mathbf{M} \right\|_\text{op} \leq \prod_{l=s}^{t-1} 1 = 1,
\label{eq:rho}
\end{equation}
which is owing to the fact that $\mathbf{P}_{l}$ is a doubly stochastic matrix and follows Lemma 3 in \cite{sun2021semi}.
Following the iterative evolution of $\bm{\overline{y}}_t$ and $\bm{y}_t^{(d)}$, we have:
\begin{align}
    \mathcal{E}_{t,2}
    &= \mathbb{E} \left[ \left\| \mathbf{Y}_{\Tilde{t}} (\mathbf{\Tilde{M}} - \mathbf{I}) \right\|_{\mathbf{\Tilde{M}}}^2 \right] \nonumber \\
    &= \mathbb{E} \left[ \left\| (\mathbf{Y}_{\Tilde{t}-1} - \eta  \mathbf{\hat{G}}_{\Tilde{t}-1} \bm{\overline{\theta}}_{\Tilde{t}-1}) \mathbf{P}_{\Tilde{t}-1} (\mathbf{\Tilde{M}} - \mathbf{I}) \right\|_{\mathbf{\Tilde{M}}}^2 \right] \nonumber \\
    &\overset{(a)}{=} \mathbb{E} \left[  \bigg\| \mathbf{Y}_{0} \mathbf{\Phi}_{0,\Tilde{t}-1} (\mathbf{\Tilde{M}} - \mathbf{I}) - \eta \sum_{s=1}^{\Tilde{t}-1}  \bm{\overline{\theta}}_s \mathbf{\hat{G}}_{s} \mathbf{\Phi}_{s,\Tilde{t}-1} (\mathbf{\Tilde{M}} - \mathbf{I}) \bigg\|_{\mathbf{\Tilde{M}}}^2 \right] \nonumber \\
    &\overset{(b)}{=} \eta^2 \mathbb{E} \left[  \bigg\| \sum_{s=1}^{\Tilde{t}-1} \bm{\overline{\theta}}_s \mathbf{\hat{G}}_{s} \mathbf{\Phi}_{s,\Tilde{t}-1} (\mathbf{\Tilde{M}} - \mathbf{I}) \bigg\|_{\mathbf{\Tilde{M}}}^2 \right] \nonumber \\
    &\overset{(c)}{\leq} 2\eta^2 \sum_{s=1}^{t-1} \mathbb{E} \left[  \left\| \bm{\overline{\theta}}_s (\mathbf{\hat{G}}_{s} - \nabla \mathbf{\hat{F}}_{s}) (\mathbf{\Phi}_{s,t-1} - \mathbf{\Tilde{M}}) \right\|_{\mathbf{\Tilde{M}}}^2 \right] \nonumber \\
    &\quad + 2\eta^2 \mathbb{E} \left[  \bigg\| \sum_{s=1}^{t-1} \bm{\overline{\theta}}_s \nabla \mathbf{\hat{F}}_{s} (\mathbf{\Phi}_{s,t-1} - \mathbf{\Tilde{M}}) \bigg\|_{\mathbf{\Tilde{M}}}^2 \right] \nonumber\\
    &\overset{(d)}{\leq} 2\eta^2 \sum_{s=1}^{t-1} \mathbb{E} \left[ \bigg\| \bm{\overline{\theta}}_s (\mathbf{\hat{G}}_{s} - \nabla \mathbf{\hat{F}}_{s}) \bigg\|_{\mathbf{\Tilde{M}}}^2 \bigg\| \mathbf{\Phi}_{s,t-1} - \mathbf{\Tilde{M}} \bigg\|_\text{op}^2 \right] \nonumber \\
    &\quad + 2\eta^2 \mathbb{E} \left[  \bigg\| \sum_{s=1}^{t-1} \bm{\overline{\theta}}_s \nabla \mathbf{\hat{F}}_{s} (\mathbf{\Phi}_{s,t-1} - \mathbf{\Tilde{M}}) \bigg\|_{\mathbf{\Tilde{M}}}^2 \right] \nonumber\\
    &\overset{(e)}{\leq} 2\eta^2 \frac{\theta_\mathrm{max}^2}{\theta_\text{min}}\sigma^2 \sum_{s=1}^{t-1} \rho_{s,t-1}^2  \nonumber \\
    &\quad + 2\eta^2 \theta_\mathrm{max}^2 \mathbb{E} \left[ \left\| \sum_{s=1}^{t-1} \nabla \mathbf{\hat{F}}_{s} (\mathbf{\Phi}_{s,t-1} - \mathbf{\Tilde{M}}) \right\|_{\mathbf{\Tilde{M}}}^2 \right] \nonumber \\
    &\overset{(f)}{\leq} 2\eta^2   \frac{\theta_\mathrm{max}^2}{\theta_\text{min}}\sigma^2 Z_{t,\rho}^{(1)} \nonumber \\
    &\quad + 2\eta^2  \theta_\mathrm{max}^2 \mathbb{E} \left[ \bigg( \sum_{s=1}^{t-1} \left\| \nabla \mathbf{\hat{F}}_{s} (\mathbf{\Phi}_{s,t-1} - \mathbf{\Tilde{M}}) \right\|_{\mathbf{\Tilde{M}}}\bigg)^2 \right] \nonumber\\
    &\overset{(g)}{\leq} 2\eta^2 \frac{\theta_\mathrm{max}^2}{\theta_\text{min}}\sigma^2 Z_{t,\rho}^{(1)} \nonumber \\
    &\quad + 2\eta^2 \theta_\mathrm{max}^2 \mathbb{E} \left[  \bigg( \sum_{s=1}^{t-1} \rho_{s,t-1} \left\| \nabla \mathbf{\hat{F}}_{s} \right\|_{\mathbf{\Tilde{M}}}\bigg)^2 \right] \nonumber \\
    &\leq 2\eta^2   \frac{\theta_\mathrm{max}^2}{\theta_\text{min}}\sigma^2 Z_{t,\rho}^{(1)} \nonumber \\
    &\quad + 2\eta^2 \theta_\mathrm{max}^2 \bigg( \sum_{s=1}^{t-1} \rho_{s,t-1} \bigg) \bigg( \sum_{s=1}^{t-1} \rho_{s,t-1} \mathbb{E}\left[  \left\| \nabla \mathbf{\hat{F}}_{s} \right\|_{\mathbf{\Tilde{M}}} \right]\bigg) \nonumber \\
    &\overset{(h)}{\leq} 2\eta^2   \frac{\theta_\mathrm{max}^2}{\theta_\text{min}}\sigma^2 Z_{t,\rho}^{(1)}  \nonumber \\
    &\quad + 2\eta^2 \theta_\mathrm{max}^2 \bigg( \sum_{s=1}^{t-1} \rho_{s,t-1} \bigg) \bigg( \sum_{s=1}^{t-1} \rho_{s,t-1} (\frac{2\kappa^2}{\theta_\text{min}} + 2Q_s) \bigg)  \nonumber \\
    &= 2\eta^2 \frac{\theta_\mathrm{max}^2}{\theta_\text{min}}\sigma^2 Z_{t,\rho}^{(1)}
    + 4\eta^2 \frac{\theta_\mathrm{max}^2}{\theta_\text{min}} \kappa^2 Z_{t,\rho}^{(2)} \nonumber \\
    &\quad + 4\eta^2 \theta_\mathrm{max}^2 \bigg( \sum_{l=1}^{t-1} \rho_{l,t-1} \bigg) \bigg( \sum_{s=1}^{t-1} \rho_{s,t-1} Q_s \bigg),
    \label{eq:term-circle-4}
\end{align}
where (a) applies the definition of $\mathbf{\Phi}_{s,\Tilde{t}-1}$. (b) holds since the initial models $\bm{y}_{0}^{(d)}$ are same. (c) follows the Jensen's inequality and $\Tilde{t}-1 \leq t-1$, (d) follows Lemma 7 in \cite{castiglia2020multi}.
Besides, we apply \eqref{eq:rho} and \eqref{eq-variance} in (e), and (f) is due to the triangle inequality $\left\|\sum_{s=1}^{t-1}\mathbf{A}_s\right\|_{\mathbf{\Tilde{M}}} \leq \sum_{s=1}^{t-1} \left\|\mathbf{A}_s\right\|_{\mathbf{\Tilde{M}}}$. Moreover, (g) follows the definition in \eqref{eq:rho} and (h) holds due to \eqref{eq:delta-F}. 
\end{proof}

%%%%%%%%%%%%%%%%%%%%%%%%%%%%Lemma 6
\begin{lemma}\label{lemma-6}
With Assumption \ref{assumptions}, we have:
\begin{equation}
\begin{split}
    \mathcal{E}_{t,3}
    & \leq \frac{\eta^2(\theta_\text{max}-1)\sigma^2}{1-2\eta^2L^2 U_3}
    + \frac{2\eta^2 U_3}{1-2\eta^2L^2 U_3} \left( 3L^2 \mathcal{E}_{t,2} \right. \\
    & \left. + 3\kappa^2 + 6\mathbb{E} \left[ \left\| \nabla F (\bm{\overline{y}}_t) \right\|^2 \right] + 6L^2 \mathcal{E}_{t,1} \right).
\end{split}
\end{equation}
\end{lemma}
\begin{proof}
First, we bound the term $ \mathbb{E} \left[ \left\| \nabla F_i(\bm{y}_{\Tilde{t}}^{(d)}) \right\|^2 \right] $ as follows:
\begin{align}
    &\quad \mathbb{E} \left[ \left\| \nabla F_i (\bm{y}_{\Tilde{t}}^{(d)}) \right\|^2 \right] \nonumber \\
    &\leq \mathbb{E} \left[ \left\| \nabla F_i (\bm{y}_{\Tilde{t}}^{(d)}) - \nabla F_i (\bm{\overline{y}}_{\Tilde{t}}) + \nabla F_i (\bm{\overline{y}}_{\Tilde{t}}) - \nabla F (\bm{\overline{y}}_{\Tilde{t}}) + \nabla F (\bm{\overline{y}}_{\Tilde{t}}) \right\|^2 \right] \nonumber\\
    &= 3\mathbb{E} \left[ \left\| \nabla F_i (\bm{y}_{\Tilde{t}}^{(d)}) - \nabla F_i (\bm{\overline{y}}_{\Tilde{t}}) \right\|^2 + 3\mathbb{E}\left\| \nabla F_i (\bm{\overline{y}}_{\Tilde{t}}) - \nabla F (\bm{\overline{y}}_{\Tilde{t}}) \right\|^2 \right] \nonumber \\
    &\quad + 3\mathbb{E} \left[ \left\| \nabla F (\bm{\overline{y}}_{\Tilde{t}}) \right\|^2 \right] \nonumber\\
    &\overset{(a)}{\leq} 3L^2 \mathbb{E} \left[ \left\| \bm{y}_{\Tilde{t}}^{(d)} - \bm{\overline{y}}_{\Tilde{t}} \right\|^2 \right] + 3\kappa^2 + 3\mathbb{E} \left[ \left\| \nabla F (\bm{\overline{y}}_{\Tilde{t}}) \right\|^2 \right] \nonumber\\
    &\leq 3L^2 \mathbb{E} \left[ \left\| \bm{y}_{\Tilde{t}}^{(d)} - \bm{\overline{y}}_{\Tilde{t}} \right\|^2 \right] + 3\kappa^2 + 6\mathbb{E} \left[ \left\| \nabla F (\bm{\overline{y}}_t) \right\|^2 \right] \nonumber \\
    &\quad + 6 \mathbb{E} \left[ \left\| \nabla F(\bm{\overline{y}}_t) - \nabla F(\bm{\overline{y}}_{\Tilde{t}}) \right\|^2 \right] \nonumber\\
    &\leq 3L^2 \mathcal{E}_{t,2} + 3\kappa^2 + 6\mathbb{E} \left[ \left\| \nabla F (\bm{\overline{y}}_t) \right\|^2 \right] + 6L^2 \mathcal{E}_{t,1},\label{eq:help-1}
\end{align}
where (a) follows \eqref{eq-smooth} and \eqref{eq-kappa} in Assumption \ref{assumptions}.
Then we have:
%%%%%% KEY %%%%%%
\begin{align}
    &\quad \mathbb{E} \left[ \left\| \bm{y}_{\Tilde{t}}^{(d)} - \bm{w}_{\Tilde{t},l}^{(i)} \right\|^2 \right] \nonumber \\
    &= \mathbb{E} \left[ \left\| \eta \sum_{s=0}^{l-1} g_i(\bm{w}_{\Tilde{t},s}^{(i)}) \right\|^2 \right] \nonumber \\
    &= \eta^2 \mathbb{E} \left[ \left\| \sum_{s=0}^{l-1} \left( g_i(\bm{w}_{\Tilde{t},s}^{(i)}) - \nabla F_i(\bm{w}_{\Tilde{t},s}^{(i)}) + \nabla F_i(\bm{w}_{\Tilde{t},s}^{(i)}) \right) \right\|^2 \right] \nonumber\\
    &\overset{(b)}{\leq} 2\eta^2 \mathbb{E} \left[ \left\| \sum_{s=0}^{l-1} \left( g_i(\bm{w}_{\Tilde{t},s}^{(i)}) - \nabla F_i (\bm{w}_{\Tilde{t},s}^{(i)}) \right) \right\|^2 \right] \nonumber \\
    &\quad + 2\eta^2 l \sum_{s=0}^{l-1} \mathbb{E} \left[ \left\| \nabla F_i (\bm{w}_{\Tilde{t},s}^{(i)}) \right\|^2 \right] \nonumber\\
    &\leq 2\eta^2 l \sigma^2 \!+\! 2\eta^2 l \! \sum_{s=0}^{l-1} \! \mathbb{E} \left[ \left\| \nabla F_i (\bm{w}_{\Tilde{t},s}^{(i)}) - \nabla F_i (\bm{y}_{\Tilde{t}}^{(d)}) + \nabla F_i (\bm{y}_{\Tilde{t}}^{(d)}) \right\|^2 \right] \nonumber\\
    &\overset{(c)}{\leq} 2\eta^2 l \sigma^2 
    + 4\eta^2L^2 l \sum_{s=0}^{\theta_i-1} \mathbb{E} \left[ \left\| \bm{y}_{\Tilde{t}}^{(d)} - \bm{w}_{\Tilde{t},s}^{(i)} \right\|^2 \right] \nonumber \\
    &\quad + 4\eta^2l \sum_{s=0}^{\theta_i-1} \mathbb{E} \left[ \left\| \nabla F_i (\bm{y}_{\Tilde{t}}^{(d)}) \right\|^2 \right] \nonumber\\
    &= 2\eta^2 l \bigg( \sigma^2 
    + 2\theta_i \mathbb{E} \left[ \left\| \nabla F_i(\bm{y}_{\Tilde{t}}^{(d)}) \right\|^2 \right] \\
    &+ 2L^2 \sum_{s=0}^{\theta_i-1} \mathbb{E} \left[ \left\| \bm{y}_{\Tilde{t}}^{(d)} - \bm{w}_{\Tilde{t},s}^{(i)} \right\|^2 \right] \bigg),
\end{align}
where (b) follows the Jensen's inequality, (c) is due to the fact that $l-1\leq \theta_i-1$ and $L$-smoothness in \eqref{eq-smooth}. Thus, summing both sides over $l=0, 1, \dots, \theta_i-1$ yields:
\begin{align}
    &\quad \sum_{l=0}^{\theta_i-1} \mathbb{E} \left[ \left\| \bm{y}_{\Tilde{t}}^{(d)} - \bm{w}_{\Tilde{t},l}^{(i)} \right\|^2 \right]\nonumber\\
    &\leq 2\eta^2 \left( \sum_{l=0}^{\theta_i-1} l \right) \left[ \sigma^2 
    + 2\theta_i \mathbb{E} \left[ \left\| \nabla F_i(\bm{y}_{\Tilde{t}}^{(d)}) \right\|^2 \right] \nonumber \right. \\
    &\quad \left. + 2L^2 \sum_{s=0}^{\theta_i-1} \mathbb{E} \left[ \left[ \left\| \bm{y}_{\Tilde{t}}^{(d)} - \bm{w}_{\Tilde{t},s}^{(i)} \right\|^2 \right] \right] \right] \nonumber\\
    &\leq \theta_\mathrm{max}(\theta_\mathrm{max}\!-\!1) \! \bigg[ \sigma^2 
    + 2\theta_i \mathbb{E} \left[ \left\| \nabla F_i (\bm{y}_{\Tilde{t}}^{(d)}) \right\|^2 \right] \nonumber \\
    &\quad + 2L^2 \!\sum_{l=0}^{\theta_i-1} \! \mathbb{E} \left[ \bigg\| \bm{y}_{\Tilde{t}}^{(d)} - \bm{w}_{\Tilde{t},l}^{(i)} \bigg\|^2 \right] \bigg].
    \label{eq:help-3}
\end{align}
We rearrange the terms in \eqref{eq:help-3}, divide both sides by $1 \!-\! 2\eta^2L^2 U_3$, and apply \eqref{eq:help-1} to obtain the following result:
% \vspace{-3pt}
\begin{align}
    &\quad\; \sum_{l=0}^{\theta_i-1} \mathbb{E}  \left[ \left\| \bm{y}_{\Tilde{t}}^{(d)} - \bm{w}_{\Tilde{t},l}^{(i)} \right\|^2 \right] \nonumber \\
    & \leq \frac{\eta^2\theta_\mathrm{max}(\theta_\mathrm{max}-1)}{1-2\eta^2L^2 U_3} \left[\sigma^2 +  2\theta_\mathrm{max} \left( 3L^2 \mathcal{E}_{t,2} \right. \right. \nonumber \\
    &\quad \left. \left. + 3\kappa^2 + 6\mathbb{E}  \left[ \left\| \nabla F (\bm{\overline{y}}_t) \right\|^2 \right] + 6L^2 \mathcal{E}_{t,1} \right) \right].
\label{eq:help-2}
\end{align}
Following the definition of $\mathcal{E}_{t,3}$, the proof is completed by summing up both sides of \eqref{eq:help-2} over $i\in \mathcal{C}_d, d\in \mathcal{D}$.
\end{proof}

\bibliographystyle{IEEEtran} % NOT {ieeetr}!
\bibliography{Reference}

\begin{IEEEbiography}[{\includegraphics[width=1in,height=1.25in,clip,keepaspectratio]{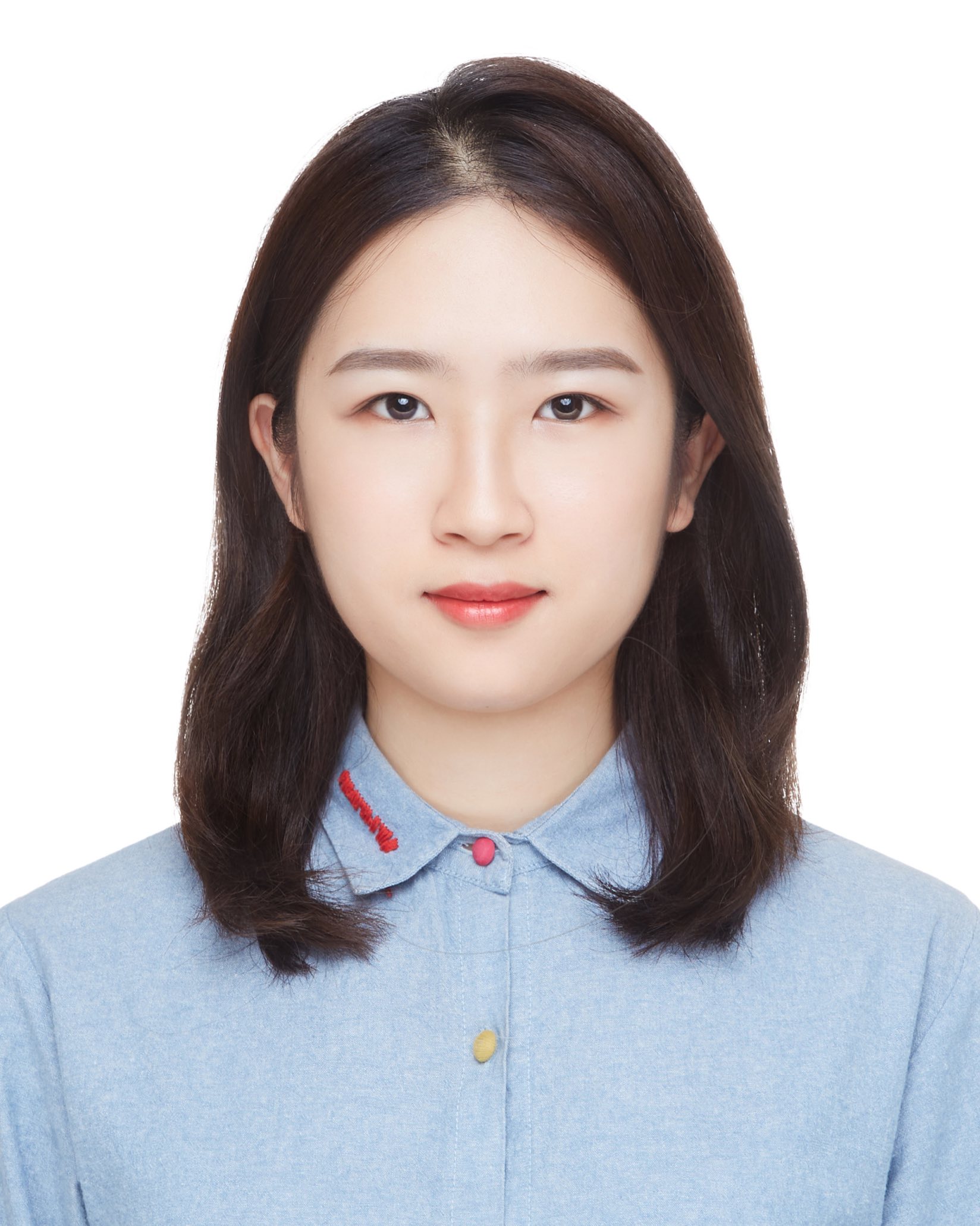}}]{Yuchang Sun}
(Graduate student member, IEEE) received the B.Eng. degree in electronic and information engineering from Beijing Institute of Technology in 2020. She is currently pursuing a Ph.D. degree at Hong Kong University of Science and Technology. Her research interests include federated learning and distributed optimization.
\end{IEEEbiography}

\begin{IEEEbiography}[{\includegraphics[width=1in,height=1.25in,clip,keepaspectratio]{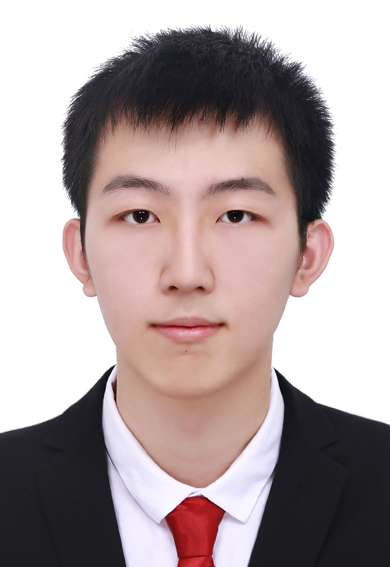}}]{Jiawei Shao}
(Graduate student member, IEEE) received the B.Eng. degree in telecommunication engineering from Beijing University of Posts and Telecommunications in 2019. He is currently pursuing a Ph.D. degree at Hong Kong University of Science and Technology. His research interests include edge intelligence and federated learning.
\end{IEEEbiography}

\begin{IEEEbiography}[{\includegraphics[width=1in,height=1.25in,clip,keepaspectratio]{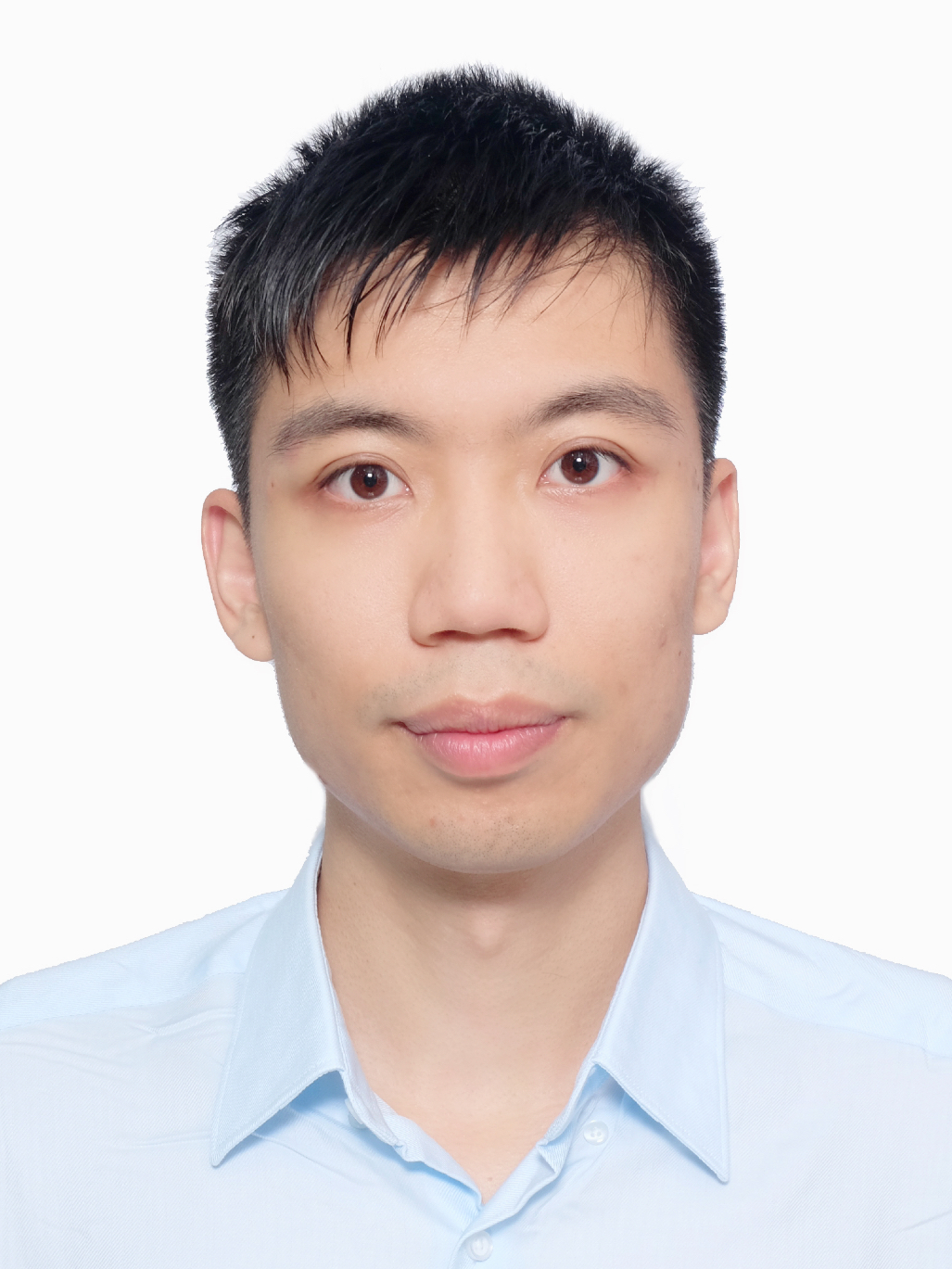}}]{Yuyi Mao}
(Member, IEEE) received the B.Eng. degree in information and communication engineering from Zhejiang University, Hangzhou, China, in 2013, and the Ph.D. degree in electronic and computer engineering from The Hong Kong University of Science and Technology, Hong Kong, in 2017. He was a Lead Engineer with the Hong Kong Applied Science and Technology Research Institute Co., Ltd., Hong Kong, and a Senior Researcher with the Theory Lab, 2012 Labs, Huawei Tech. Investment Co., Ltd., Hong Kong. He is currently a Research Assistant Professor with the Department of Electronic and Information Engineering, The Hong Kong Polytechnic University, Hong Kong. His research interests include wireless communications and networking, mobile-edge computing and learning, and wireless artificial intelligence.

He was the recipient of the 2021 IEEE Communications Society Best Survey Paper Award and the 2019 IEEE Communications Society and Information Theory Society Joint Paper Award. He was also recognized as an Exemplary Reviewer of the IEEE Wireless Communications Letters in 2021 and 2019 and the IEEE Transactions on Communications in 2020.
\end{IEEEbiography}

\begin{IEEEbiography}[{\includegraphics[width=1in,height=1.25in,clip,keepaspectratio]{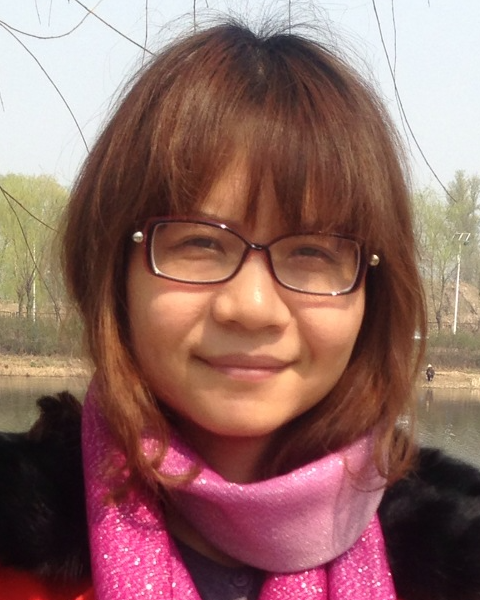}}]{Jessie Hui Wang}
(Member, IEEE) {received the Ph.D. degree in information engineering from The Chinese
University of Hong Kong in 2007. Before that, she received the B.S. degree and the M.S. degree in computer science from Tsinghua University. She is currently a tenured Associate Professor with Tsinghua University. Her research interests include Internet routing, distributed computing, network measurement and Internet economics.}
\end{IEEEbiography}

\begin{IEEEbiography}[{\includegraphics[width=1in,height=1.25in,clip,keepaspectratio]{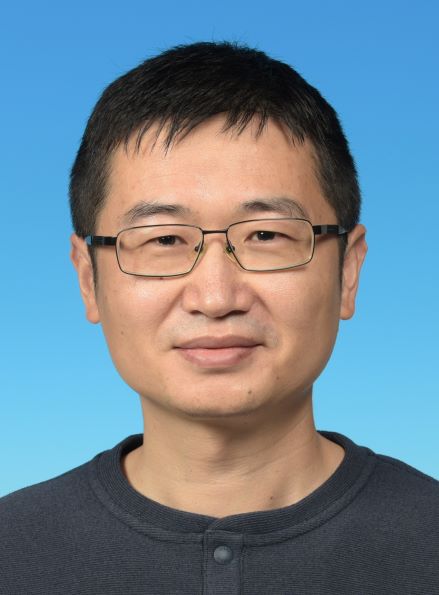}}]{Jun Zhang}

(Fellow, IEEE) received the B.Eng. degree in Electronic Engineering from the University of Science and Technology of China in 2004, the M.Phil. degree in Information Engineering from the Chinese University of Hong Kong in 2006, and the Ph.D. degree in Electrical and Computer Engineering from the University of Texas at Austin in 2009. He is an Associate Professor in the Department of Electronic and Computer Engineering at the Hong Kong University of Science and Technology. His research interests include wireless communications and networking, mobile edge computing and edge AI, and cooperative AI.

Dr. Zhang co-authored the book Fundamentals of LTE (Prentice-Hall, 2010). He is a co-recipient of several best paper awards, including the 2021 Best Survey Paper Award of the IEEE Communications Society, the 2019 IEEE Communications Society \& Information Theory Society Joint Paper Award, and the 2016 Marconi Prize Paper Award in Wireless Communications. Two papers he co-authored received the Young Author Best Paper Award of the IEEE Signal Processing Society in 2016 and 2018, respectively. He also received the 2016 IEEE ComSoc Asia-Pacific Best Young Researcher Award. He is an Editor of IEEE Transactions on Communications, IEEE Transactions on Machine Learning in Communications and Networking, and was an editor of IEEE Transactions on Wireless Communications (2015-2020). He served as a MAC track co-chair for IEEE Wireless Communications and Networking Conference (WCNC) 2011 and a co-chair for the Wireless Communications Symposium of IEEE International Conference on Communications (ICC) 2021. He is an IEEE Fellow and an IEEE ComSoc Distinguished Lecturer.

\end{IEEEbiography}

\end{document}